\documentclass{article}

     \PassOptionsToPackage{numbers, compress}{natbib}

\usepackage[final]{neurips_2022}

\usepackage[intlimits]{amsmath}
\usepackage{amssymb}
\usepackage{amsfonts}
\usepackage{amsthm}
\usepackage{graphics}
\usepackage{graphicx}
\usepackage{setspace}
\usepackage{hyperref}
\usepackage{cleveref}
\usepackage{wrapfig}
\usepackage{subcaption}
\usepackage{algorithm}
\usepackage{algorithmic}
\newtheorem{theorem}{\protect\theoremname}
  \newtheorem{lem}[theorem]{\protect\lemmaname}
  \newtheorem{corollary}[theorem]{\protect\corrolaryname}
  \newtheorem{defn}[theorem]{\protect\definitionname}

  \newtheorem{remark}{Remark}

    \providecommand{\definitionname}{Definition}
  \providecommand{\examplename}{Example}
  \providecommand{\lemmaname}{Lemma}
  \providecommand{\corrolaryname}{Corollary}
  \providecommand{\propositionname}{Proposition}
  \providecommand{\conditionsname}{Conditions}
\providecommand{\theoremname}{Theorem}
\providecommand{\assumptionname}{Assumption}

\def \M{\mathcal{M}}
\def \S{\mathcal{S}}
\def \A{\mathcal{A}}
\def \R{\mathbb{R}}

\usepackage[utf8]{inputenc} 
\usepackage[T1]{fontenc}    
\usepackage{url}            
\usepackage{booktabs}       
\usepackage{amsfonts}       
\usepackage{nicefrac}       
\usepackage{microtype}      
\usepackage{xcolor}         

\title{Identifiability and Generalizability from Multiple Experts in Inverse Reinforcement Learning}

\author{%
  Paul Rolland \\
  LIONS, EPFL\\Lausanne, Switzerland\\
  \texttt{paul.rolland@epfl.ch} \\
  \And
  Luca Viano \\
  LIONS, EPFL\\
  Lausanne, Switzerland\\
  \texttt{luca.viano@epfl.ch} \\
  \And
  Norman Schürhoff \\
  SFI, UNIL \\
  Lausanne, Switzerland\\
  \texttt{norman.schuerhoff@unil.ch} \\
  \And
  Boris Nikolov \\
  SFI, UNIL \\
  Lausanne, Switzerland\\
  \texttt{boris.nikolov@unil.ch}\\
  \And
  Volkan Cevher \\
  LIONS, EPFL \\
  Lausanne, Switzerland\\
  \texttt{volkan.cevher@epfl.ch}\\
}

\begin{document}

\maketitle

\begin{abstract}
  While Reinforcement Learning (RL) aims to train an agent from a reward function in a given environment, Inverse Reinforcement Learning (IRL) seeks to recover the reward function from observing an expert's behavior. It is well known that, in general, various reward functions can lead to the same optimal policy, and hence, IRL is ill-defined. However, \cite{cao2021identifiability} showed that, if we observe two or more experts with different discount factors or acting in different environments, the reward function can under certain conditions be identified up to a constant. This work starts by showing an equivalent identifiability statement from multiple experts in tabular MDPs based on a rank condition, which is easily verifiable and is shown to be also necessary. We then extend our result to various different scenarios, i.e., we characterize reward identifiability in the case where the reward function can be represented as a linear combination of given features, making it more interpretable, or when we have access to approximate transition matrices.
  Even when the reward is not identifiable, we provide conditions characterizing when data on multiple experts in a given environment
  allows to generalize and train an optimal agent in a new environment.
  Our theoretical results on reward identifiability and generalizability are validated in various numerical experiments.

  
\end{abstract}

\section{Introduction}

Engineering a reward function in Reinforcement Learning can be troublesome in certain scenarios like driving \cite{Knox:2021}, robotics \cite{osa2018algorithmic}, and economics/finance \cite{Charpentier:2020}. 
In economics and finance, the reward or objective/utility function of the agent are of fundamental importance but are not known a priori \cite{vonneumann1947, Pratt:1964, arrow1965aspects, Kahneman:Tversky:1979}.
In such cases, it may be easier to get demonstrations from an expert policy. Therefore, multiple algorithms have been developed to learn from demonstrations, e.g., in inverse reinforcement learning (IRL) and imitation learning (IL).

In IRL, the goal is to recover the reward function maximized by the agent, while in IL the expert demonstrations are used solely to learn a nearly optimal policy.
In economics/finance, inference on the reward function is the focus of a large literature on estimation, testing, and policy analysis of structural models \cite{Hansen:1982, Rust:1987, Hotz:Miller:1993}. However, the reward function is often highly parameterized and represented by a low-dimensional set of parameters, or the literature focuses on estimating reduced-form causal relationships but not the true reward function \cite{Heckman:1979, Angrist:Imbens:Rubin:1996}. 
The attractiveness of IRL relies on the fact that the reward function is the most ``succinct'' representation of a task \cite{Sutton:1998}. Indeed, identifying the reward function for each state-action pair allows generalizing the task to different transition dynamics and environments, which is not possible when using IL or highly parameterized structural models.

However, the IRL problem is unfortunately ill-posed since there always exist infinitely many reward functions for which the observed expert policy is optimal \cite{Russell:1998, Ng:2000}. The problem is known as reward shaping, and it is intuitively explained with the fact that, in the long term, the optimal policy is not affected by inflating the reward in the current period and decreasing the one in the next.
This difficulty originated a long debate on advantages and disadvantages of IL and IRL \cite{piot2013learning, piot2016bridging, Ho:2016, Fu:2018}.

When multiple experts are available, differing in the transition matrices of the environments they each act in, and/or their discount factors, IRL can in certain cases infer the true reward function, up to a constant \cite{Amin:2016, Amin:2017, Likmeta:2021, cao2021identifiability}. 
Inspired by \cite{cao2021identifiability}, we derive an equivalent necessary and sufficient condition on the expert environments, which is easily verifiable, ensuring that the true reward can be identified up to a constant shift. When this identifiability condition holds, the state-action dependent rewards can be recovered from expert demonstrations.
We then derive identifiability results in various alternative scenarios, e.g., when we only have access to approximate transition matrices and, alternatively, when the reward function is known to be a linear combination of given features \cite{Devidze:2021, Jarrett:2021}.

However, full reward identifiability remains a strong requirement, and we provide a negative result of non-identifiability from any number of experts, in the presence of exogenous variables in the MDP. Nonetheless, even when the identifiability condition does not hold, the recovered reward function could still be used to train an optimal expert for a different environment. To this end, we characterize situations where observing multiple experts in given environments allows to train an optimal agent in a new environment.


\section{Related work}

Since its introduction in \cite{Russell:1998,Ng:2000}, the IRL problem has been known to be ill-posed, since the observed expert policy can be optimal with respect to various reward functions. 
The set of reward transformations that preserve policy optimality are studied in \cite{Ng:1999, Ng:2000, cao2021identifiability, Skalse:2021, Gleave:2020}.
\cite{Mindermann:2018} studied the unidentifiability related to suboptimal experts.

In this paper, we assume access to the optimal entropy regularized policies of multiple experts.
Significant progress has been made to construct heuristics that select a single reward function from the set of IRL solutions (often called the feasible set), such as feature-based matching \cite{Abbeel:2004}, maximum margin IRL \cite{Ratliff:2006}, maximum causal entropy IRL \cite{Ziebart:2008,ziebart2010modeling}, maximum relative entropy IRL \cite{boularias2011relative}, Bayesian IRL \cite{balakrishnan2020efficient, Ramachandran:2007, Brown:2020d}, first-order optimality conditions \cite{Pirotta:2016, Ramponi:2020} or second-order optimality conditions \cite{Hwang:2022, Metelli:2017}. Popular IL algorithms implicitly select a feasible reward function via a convex reward regularizer \cite{Ho:2016, gangwani2020stateonly, torabi2018generative} or using preference/ranking based algorithms \cite{Brown:2019, Brown:2020c}. However, none of these approaches guarantee the identification of the true reward function.

The problem of identifiability in IRL has been investigated first in \cite{Amin:2016, Amin:2017} that study a setting where the learner can actively select optimal experts in multiple environments. The main result in \cite{Amin:2016, Amin:2017} is that interactively querying environments outputs a reward within a constant shift from the true one. The multiple experts setting has also been studied in \cite{Brown:2021} but in the context of value alignment verification where the aim is not to recover the reward function but rather verify that the value function of the agent is close to a target value.
IRL from multiple MDPs also appears in \cite{Likmeta:2021} where the authors consider the problem of learning a reward function compatible with a dataset of demonstrations collected by multiple experts.
In addition, \cite{kim2021reward} study structural conditions on the MDP for reward identification in the finite horizon setting and \cite{Dvijotham:2010} study identifiability in linearly solvable MDPs.

Our work is inspired by \cite{cao2021identifiability}. Our first identifiability result provides an equivalent statement as their \emph{value distinguishability} condition, but can be easily checked in practice, and allows to derive other identifiability results in alternative scenarios. 
Finally, the motivation for IRL is often predicting the expert behavior under new transitions dynamics \cite{abbeel2004apprenticeship, Levine:2011, Fu:2018}. We show that for this goal, it is not necessary to identify the exact reward, hence we give a condition on the observed experts' environments and the test environment under which an optimal expert can be trained in the test environment. This perspective has also been taken in \cite{Metelli:2021}. However, this work requires stronger assumptions on the transfer environment that we avoid in this paper, only requiring access to multiple experts.
Moreover, our work contributes to AI safety \cite{Everitt:2016, Amodei:2016, Leike:2018} alleviating the \emph{reward hacking} and \emph{side effects} problems \cite{Amodei:2016}. Indeed, by restricting the reward to linear combinations of a set of chosen features, we can provably recover an interpretable reward function inducing the optimal behavior, which is particularly desirable in medical applications \cite{srinivasan2020interpretable,jarrett2021inverse}.


An important consideration for IRL comes from \cite{Abel:2021} that formalizes the fact that there exist tasks that can not be induced by optimizing a reward function. In this work and in IRL in general, we bypass this difficulty assuming that the expert is optimizing a reward function. 

\subsection{Related works in the economics literature}
The economics/finance literature differentiates between axiomatic and revealed preference theory. In axiomatic preference theory, the reward function is posited or derived from basic axioms. In empirical and experimental work, however, simple reward function specifications are often rejected and agents have been shown to exhibit behavioral biases and/or non-standard preferences.

Differently, our work relates to the literature on revealed preference. Revealed preference theory, initiated by \cite{paul1938note, samuelson1948consumption}, provides an approach to analyze actions (e.g., consumer’s demand or investors' trading) by assuming that observed choices provide information about the underlying preferences, or reward function. Revealed preference theory is, hence, similar in spirit to IRL. But IRL has not widely been used in revealed preference analysis. We refer to \cite{demuynck2019samuelson, echenique2020new} for excellent reviews of recent advances in revealed preference theory.
The goal of revealed preference theory is to recover the agents’ preferences. This task is important because knowledge of the reward function is required to conduct counterfactual policy analysis. Notice that for this task, knowing only the policy function is insufficient. In financial applications, for instance, the impact of a Tobin tax can be assessed only knowing investors’ preferences for trading (see, e.g. \cite{tobin1978proposal}).
\section{Preliminaries}


A typical RL environment is characterised by a Markov Decision Process $\M = \{\S, \A, T, \gamma, r, P_0\}$, where $\S, \A$ are the sets of states and actions respectively, $T: \S \times \A \times \S \rightarrow [0,1]$ is the state transition probability, i.e., $T(s' | s,a)$ denotes the probability of arriving in state $s'$ when taking action $a$ in state $s$. $R: \S \times \A \rightarrow \R$ denotes the reward function, $\gamma$ the discount factor and $P_0$ is the initial state distribution. At each time step $t$, an agent observes the current state $s_t \in \S$ and takes an action $a_t \sim \pi(\cdot | s_t)$ where $\pi$ is the agent's policy which determines a distribution over all actions in $\A$ at every state. The agent gets a reward $r_t = r(s_t, a_t)$ and transitions to a new state $s_{t+1}$ sampled according to the transition probability $T$.

An agent acting optimally in $\M$ seeks to maximize its cumulative sum of rewards. In addition, we assume that the agent seeks to diversify its possible actions, and hence that it maximizes the following entropy regularized sum of discounted rewards:
\begin{equation}
\label{eq:1}
    V_\lambda^\pi(s) = \mathbb{E}_s^\pi\left[\sum_{t=0}^\infty (\gamma^t (r(s_t, a_t) + \lambda \mathcal{H}(\pi(\cdot | s_t))))\right],
\end{equation}

where $\mathbb{E}_s^\pi$ denotes the expectation over trajectories $\{(s_t,a_t\}_{t \geq 0}$ starting from state $s_0 = s$ and following policy $\pi$ and $\mathcal{H}(\pi) = -\sum_{a\in \A} \pi(a) \log \pi(a)$ is the entropy of $\pi$. The function $V_\lambda^\pi$ is called the (entropy regularized) value function of $\pi$.

In Inverse RL, the reward function $r$ is unknown, but we observe an agent acting optimally with respect to some reward function, and we wish to recover the reward function that the agent optimizes. We now recall some results from \cite{cao2021identifiability}.

\begin{theorem} \label{thm:single_expert}
For a fixed policy $\pi(a|s) > 0$, discount factor $\gamma \in [0, 1)$, and an arbitrary choice of function $v:\S \rightarrow \R$, there is a unique corresponding reward function 
\[
r(s,a) = \lambda \log \pi(a|s) - \gamma \sum_{s' \in \S} T(s'|s,a) v(s') + v(s)
\]
such that the MDP with reward $r$ yields an entropy-regularized optimal policy $\pi_\lambda^* = \pi$ and $V_\lambda^\pi = v$.
\end{theorem}

By observing a single expert, it is hence possible to design a reward that yields any arbitrary value function, and there are hence $|\S|$ degrees of freedom remaining in the recovered reward function. An idea explored in~\cite{cao2021identifiability} is to assume that we observe two experts in two different MDPs with different transition dynamics and discount rates, but acting optimally with respect to the same reward function. The authors show that the reward can be identified up to a constant from observing the expert policies provided that the MDPs of the experts satisfy the following \textit{value-distinguishing} assumption.

\begin{defn} \label{def:distinguishability}
Consider a pair of Markov decision problems on the same state and action spaces, but with respective discount rates $\gamma_1, \gamma_2$ and transition probabilities $T^1, T^2$. We say that this pair is value-distinguishing if, for any function $v^1, v^2:\S \rightarrow \R$, the statement
\begin{equation} \label{eq:disting}
    v^1(s) - \gamma_1 \sum_{s'\in\S} T^1(s'|s,a) v^1(s') = v^2(s) - \gamma_2 \sum_{s'\in\S} T^2(s'|s,a) v^2(s') \ for \ all \ a\in\A, s\in\S
\end{equation}

implies at least one of $v^1$ and $v^2$ is a constant function.
\end{defn}

The way this assumption is stated makes it difficult to verify in practice, and the authors of ~\cite{cao2021identifiability} do not attempt to verify it in their experiments.

\section{Reward identification and generalization}

In this section, we present our main theoretical results on reward identifiability and generalizability. In the first part, we show an equivalent condition to Definition~\ref{def:distinguishability} for reward identification from two experts (Theorem~\ref{thm:identification}). The simplicity of our condition makes it easily verifiable and extendable to various scenarios, in particular to the cases where we observe more than two experts (Corollary~\ref{cor:2}), when the class of rewards is linearly parameterized with a set of given features (Theorem~\ref{thm:param}), or when we have access to approximated transition matrices (Theorem~\ref{thm:prob}). We also provide a negative result on reward non-identifiability in MDPs with exogenous variables, which are common in many real world scenarios. In the second part, we analyse reward generalizability. Here, we provide a condition guaranteeing that a reward compatible with two experts leads to an optimal policy in a third environment (Theorem~\ref{thm:generalization}). The proofs of the results are all postponed to Appendix~\ref{app:A}.

\subsection{Reward identifiability}

Consider two Markov decision problems on the same set of states and actions $\S$ and $\A$ respectively, but with different transition dynamics $T^1, T^2$ and discount factors $\gamma_1, \gamma_2$. Let $r \in \R^{|\S|\times |\A|}$ be the reward function common to the two environments, and let $v^1, v^2 \in \R^{|\S|}$ be the entropy regularized values functions associated expert policies $\pi^1$ and $\pi^2$ in each environment respectively. According to Theorem~\ref{thm:single_expert}, we have that $\forall (s,a)\in\S\times\A$,
\begin{align*}
    r(s,a) &= \lambda \log \pi^1(a|s) - \gamma_1 \sum_{s' \in \S} T^1(s'|s,a) v^1(s') + v^1(s) \\
    &= \lambda \log \pi^2(a|s) - \gamma_2 \sum_{s' \in \S} T^2(s'|s,a) v^2(s') + v^2(s).
\end{align*}

We hence deduce that $\forall a\in\A$,
\begin{equation}
    \begin{pmatrix}
    I - \gamma_1 T_a^1 & -(I - \gamma_2 T_a^2)
    \end{pmatrix}
    \begin{pmatrix}
    v^1 \\
    v^2
    \end{pmatrix}
    = 
    \lambda \log \pi^2(\cdot|a) - \lambda \log \pi^1(\cdot|a),
\end{equation}

where $\forall a\in\A$, $T_a^i \in \R^{\S\times\S}$ is the transition matrix for action $a$ and expert $i=1,2$, i.e., $T_a^i(s,s') = T^i(s'|s,a)$. By including all available actions to the experts, we can write
\begin{equation} \label{eq:A1}
    \begin{pmatrix}
I - \gamma_1 T_{a_1}^1 & -(I - \gamma_2 T_{a_1}^2) \\
\vdots & \vdots \\
I - \gamma_1 T_{a_{|\mathcal{A}}|}^1 & -(I - \gamma_2 T_{a_{|\mathcal{A}|}}^2)
\end{pmatrix}
\begin{pmatrix}
    v^1 \\
    v^2
    \end{pmatrix}
    =
    \begin{pmatrix}
    \lambda \log \pi^2(\cdot|a_1) - \lambda \log \pi^1(\cdot|a_1) \\
    \vdots \\
    \lambda \log \pi^2(\cdot|a_{|\A|}) - \lambda \log \pi^1(\cdot|a_{|\A|})
    \end{pmatrix}.
\end{equation}

In order to identify a unique reward function, we need to identify a unique associated value function. We hence want the linear system~\eqref{eq:A1} to yield a unique solution, i.e., the $|\A||\S|\times 2|\S|$ matrix on the left hand side to be full rank, i.e., to have rank $2|\S|$. However, it is well known that, for any MDP, adding a constant to the reward would not change the associated optimal policy. Hence, there is an intrinsic degree of freedom in reward identifiability which is impossible to get rid of from only observing expert policies. In order to identify the reward up to a constant, we need this degree of freedom to be the only one in the linear system~\eqref{eq:A1}, i.e., the associated matrix to have rank $2|\S|-1$. This result is summarized in the following theorem, and its complete proof can be found in Appendix~\ref{app:A1}. 

\begin{theorem}
\label{thm:identification}
Consider two Markov decision problems on the same set of states and actions, but with different transition dynamics $T_1, T_2$ and discount factors $\gamma_1, \gamma_2$. Suppose that we observe two experts acting each in one of these environments, optimally with respect to the same reward function, in the sense that their policies maximize the entropy regularized reward in their respective environments. Then, the reward function can be recovered up to the addition of a constant if and only if
\begin{equation}\label{eq:7.1}
\text{rank}
\begin{pmatrix}
I - \gamma_1 T_{a_1}^1 & I - \gamma_2 T_{a_1}^2 \\
\vdots & \vdots \\
I - \gamma_1 T_{a_{|\mathcal{A}}|}^1 & I - \gamma_2 T_{a_{|\mathcal{A}|}}^2
\end{pmatrix}
= 2|\mathcal{S}| - 1.
\end{equation}
\end{theorem}

This condition turns out to be equivalent to Definition~\ref{def:distinguishability}, as shown at the end of Appendix~\ref{app:A1}, but is stated in a way that is easier to check in practice and allows us to further characterize identifiability in various scenarios. First of all, this result naturally extends to the case where we observe any number of experts. We provide hereafter the result in the case of three experts.

\begin{corollary}\label{cor:3experts}
Consider three Markov decision problems on the same set of states and actions, but with different transition dynamics $T_1, T_2, T_3$ and discount factors $\gamma_1, \gamma_2, \gamma_3$. Suppose that we observe three experts acting each in one of these environments, optimally with respect to the same reward function. Then, the reward function can be recovered up to the addition of a constant if and only if
\begin{equation}
\text{rank}
\begin{pmatrix}
I - \gamma_1 T_{a_1}^1 & I - \gamma_2 T_{a_1}^2 & \textbf{0} \\
\vdots & \vdots & \vdots \\
I - \gamma_1 T_{a_{|\mathcal{A}|}}^1 & I - \gamma_2 T_{a_{|\mathcal{A}|}}^2 & \textbf{0} \\
I - \gamma_1 T_{a_1}^1 & \textbf{0} & I - \gamma_3 T_{a_1}^3 \\
\vdots & \vdots & \vdots \\
I - \gamma_1 T_{a_{|\mathcal{A}|}}^1 & \textbf{0} & I - \gamma_3 T_{a_{|\mathcal{A}|}}^3
\end{pmatrix}
= 3|\mathcal{S}| - 1.
\end{equation}
\end{corollary}

An interesting scenario is the one where the two experts act in the same environment, and only the discount rate is varied.

\begin{corollary}\label{cor:2}
Consider two Markov decision problems on the same set of states and actions, with the same transition matrix $T$ and reward function but different discount factors $\gamma_1 \neq \gamma_2$. Then, the reward function is identifiable up to a constant by observing two experts in $(T, \gamma_1), (T, \gamma_2)$ iff
\begin{equation} \label{eq:7.1.2}
    \text{rank}
    \begin{pmatrix}
    T_{a_1} - T_{a_2} \\
    \vdots \\
    T_{a_1} - T_{a_{|\mathcal{A}|}}
    \end{pmatrix}
    = |\mathcal{S}| - 1.
\end{equation}
\end{corollary}

\begin{remark}
Interestingly, condition~\eqref{eq:7.1.2} is equivalent to the condition for identification of a action-independent reward from a single expert, assuming such a reward exists (\cite{cao2021identifiability}, Corollary 3).
\end{remark}

Next, we provide a negative result concerning MDPs with exogenous variables, i.e., a variable whose dynamics are independent of the agent's action. This MDP class is common in economics/finance and has been studied in many real world scenarios including inventory control problems \cite{Joshi:2013}, variable weather conditions and customer demands \cite{Dietterich:2018}, wildfire management \cite{McGregor:2017}, and stock market fluctuations \cite{Liu:2021}. We also provide examples involving such variables in the experimental section.

\begin{corollary}\label{cor:exo}
Suppose that the state space is constructed as a set of variables each taking a finite number of values, i.e., $\S = \{s\in \R^d: s_i\in \S_i\}$. The transition matrices for each action $a$ can be defined by specifying the evolution of each state variable $s_i^{t+1}$ depending on $(s^t, a)$. Suppose that there exists a state variable whose evolution only depends on its previous value, but neither on the other state variables nor the action taken: such a variable is called an \textbf{exogenous} variable. Note that this variable can still affect the evolution of all other variables, and its evolution can vary across the environment of the observed experts. Then, the reward function is \textbf{not} identifiable (even up to a constant) using any number of experts.
\end{corollary}

Such a negative result motivates the search for milder requirements than arbitrary reward identification, which is too hard of a goal to achieve in certain scenarios.

A possible way to improve reward identifiability is to restrict the class of possible rewards, e.g., by constraining it to be a linear combination of a set of chosen features. This is known as Feature matching IRL~\cite{abbeel2004apprenticeship, ng2000algorithms, ratliff2006maximum, syed2008apprenticeship, heim2019practitioner}. The smaller the set of features, the easier to identify the reward, as described in the following theorem. This method also allows to recover a more interpretable reward function, since the recovered parameters are associated with specific features.

\begin{theorem} \label{thm:param}
Suppose that we restrict the class of possible reward functions to the one parameterized as $r_w(s,a) = w^T f_{s,a}$ $\forall a\in \A, s \in \S$ where $f:\S\times \A \rightarrow \R^d$ is a given feature function, and $w \in \R^d$ denotes the reward parameters. Suppose that the $d$ chosen features are linearly independent, i.e., that $f_{s,a}^T v = 0 \ \forall s,a \Rightarrow v = 0$. Then, if $\mathbf{1} \in \text{Im} \begin{pmatrix} f_{a_1} \\ \vdots \\ f_{a_{|\A|}} \end{pmatrix}$, the reward is identifiable up to constant by observing experts acting in $(T^1, \gamma_1), (T^2, \gamma_2)$ if and only if
\begin{equation}\label{eq:paramR}
\text{rank}
\begin{pmatrix}
I - \gamma_1 T_{a_1}^1 & I - \gamma_2 T_{a_1}^2 & \textbf{0} \\
\vdots & \vdots & \vdots \\
I - \gamma_1 T_{a_{|\mathcal{A}|}}^1 & I - \gamma_2 T_{a_{|\mathcal{A}|}}^2 & \textbf{0} \\
I - \gamma_1 T_{a_1}^1 & \textbf{0} & f_{a_1} \\
\vdots & \vdots & \vdots \\
I - \gamma_1 T_{a_{|\mathcal{A}|}}^1 & \textbf{0} & f_{a_{|\A|}}
\end{pmatrix}
= 2|\mathcal{S}| + d - 1.
\end{equation}
where $f_a = (f_{s_1,a} \ldots f_{s_{|\S|},a})^T \in \R^{|\S|\times d}$. On the other hand, if $\mathbf{1} \notin \text{Im} \begin{pmatrix} f_{a_1} \\ \vdots \\ f_{a_{|\A|}} \end{pmatrix}$, then the reward can be exactly recovered provided that the rank of the matrix on the left hand side of equation~\eqref{eq:paramR}, which augments equation~\eqref{eq:7.1} by the features being matched, is $2|\mathcal{S}| + d$.
\end{theorem}


Finally, it usually happens that the exact transition matrices $\{T_a\}_{a\in \A}$ are not known exactly and must be estimated, e.g., from samples. Verifying condition~\eqref{eq:7.1} on the approximated matrices may be misleading since the rank is very sensitive to small perturbations. Hence, we provide hereafter an identifiability condition in the case where we only have access to approximated transition matrices.

\begin{theorem}\label{thm:prob}
Suppose that we approximate the transition matrices $\{T_a^i\}_{a\in \A}$ as $\{\hat{T}^i_a\}_{a\in \A}$ such that $\|T_a^i - \hat{T}^i_a\|_2 \leq \epsilon$ $\forall a\in \A$, $i=1,2$. Suppose that we verify condition~\eqref{eq:7.1} using the approximated matrices, i.e., we compute the second smallest eigenvalue $\sigma$ of the following matrix:

\begin{equation}
\begin{pmatrix}
I - \gamma_1 \hat{T}_{a_1}^1 & I - \gamma_2 \hat{T}_{a_1}^2 \\
\vdots & \vdots \\
I - \gamma_1 \hat{T}_{a_{|\mathcal{A}}|}^1 & I - \gamma_2 \hat{T}_{a_{|\mathcal{A}|}}^2
\end{pmatrix}.
\end{equation}

Then, condition~\eqref{eq:7.1} on the true transition matrices $\{T_a\}_{a\in \A}$ holds provided that
\begin{equation}
\sigma > \epsilon \sqrt{2|\A|} \max(\gamma_1, \gamma_2).
\end{equation}
\end{theorem}
\begin{remark}
The matrix estimator $\hat{T}_a$ can be obtained from samples. For example, \cite{Kearns:1998}[Lemma 5] shows that a high probability bound on the max norm $\|T_a - \hat{T}_a\|_{\mathrm{max}}\leq \epsilon$ requires $\mathcal{O}(\epsilon^{-4})$ samples from a generative model \cite{Azar:2012}. This would imply the following bound on the spectral norm: $\|T_a - \hat{T}_a\|_{2} \leq |{\mathcal{S}}|\|T_a - \hat{T}_a\|_{\mathrm{max}}\leq |{\mathcal{S}}|\epsilon$.
However, the dependence on $\epsilon$ can be improved as we show next applying the matrix Bernstein bound \cite{Hsu:2012, Hsu:2011}.
\end{remark}
\begin{theorem}\label{thm:bernstein}
Let $\hat{T}_a$ be the empirical estimator for $T_a$. Then with probability greater than $1-\delta$,
\begin{equation}
    \|T_a - \hat{T}_a\|_{2} \leq |\mathcal{S}|\sqrt{\frac{\log{\frac{|\mathcal{S}||\mathcal{A}|}{\delta}}}{2N}} + \frac{2(|\mathcal{S}|+1)\log{\frac{|\mathcal{S}||\mathcal{A}|}{\delta}}}{3N} \quad \forall a\in\A.
\end{equation}
Therefore, we can obtain $\|T_a - \hat{T}_a\|_{2} \leq \epsilon$ with $\mathcal{O}(\epsilon^{-2})$ samples.
\end{theorem}

\subsection{Generalization to unknown environments}

We now focus on reward generalizability, i.e., the ability to recover a reward function that would allow us to train an optimal policy in a new environment. 
Suppose that we recover a reward function that is compatible with two experts acting in two MDPs $\M_1, \M_2$, and that we use this reward to train an expert in a third environment $\M_3$, assuming all environments share the same true reward function but possibly different transition dynamics and discount factors. What condition guarantees that the trained expert will be optimal in $\M_3$?

This generalization requirement is milder than full reward identification. Indeed, being able to identify the reward (even up to a constant) naturally allows to train an optimal policy in any other environment sharing the same reward. However, even in the presence of non-trivial degrees of freedom, it may be the case that any recovered reward suffices to train an optimal policy in a given other environment.

Intuitively, the third training environment should not vary too much from the observed environments $\M_1, \M_2$. More precisely, if observing a third expert in environment $3$ does not provide any further identification of the reward than with environments $1$ and $2$, then any reward compatible with environments $1$ and $2$ leads to an optimal policy in environment $3$. The condition is made precise in the following theorem.

\begin{defn}
Consider three Markov decision problems on the same set of states and actions, but with different transition matrices $T_1, T_2, T_3$ and discount factors $\gamma_1, \gamma_2, \gamma_3$. Suppose that we observe two optimal entropy regularized experts with respect to the same reward function in environments $1$ and $2$. We say that $(T^1, \gamma_1), (T^2, \gamma_2)$ \textbf{generalize to} $(T^3, \gamma_3)$ if any reward compatible with the two experts in environments $1$ and $2$ leads to an optimal expert in environment $3$. The definition naturally extends to more than two observed experts.
\end{defn}

\begin{theorem}
\label{thm:generalization}

$(T^1, \gamma_1), (T^2, \gamma_2)$ generalize to $(T^3, \gamma_3)$ if and only if

\begin{equation}\label{eq:7.2}
\text{rank}
\begin{pmatrix}
I - \gamma_1 T_{a_1}^1 & I - \gamma_2 T_{a_1}^2 \\
\vdots & \vdots \\
I - \gamma_1 T_{a_{|\mathcal{A}|}}^1 & I - \gamma_2 T_{a_{|\mathcal{A}|}}^2
\end{pmatrix}
= \text{rank}
\begin{pmatrix}
I - \gamma_1 T_{a_1}^1 & I - \gamma_2 T_{a_1}^2 & \textbf{0} \\
\vdots & \vdots & \vdots \\
I - \gamma_1 T_{a_{|\mathcal{A}|}}^1 & I - \gamma_2 T_{a_{|\mathcal{A}|}}^2 & \textbf{0} \\
I - \gamma_1 T_{a_1}^1 & \textbf{0} & I - \gamma_3 T_{a_1}^3 \\
\vdots & \vdots & \vdots \\
I - \gamma_1 T_{a_{|\mathcal{A}|}}^1 & \textbf{0} & I - \gamma_3 T_{a_{|\mathcal{A}|}}^3
\end{pmatrix} - |\mathcal{S}|.
\end{equation}

This condition is also necessary, in the sense that, if it does not hold, then there exists a reward function compatible with experts $1$ and $2$ but which leads to a sub-optimal policy in environment $3$.
\end{theorem}

One interesting question is whether observing two experts in the same environment with different discount factors allows to generalize to any other expert with arbitrary discount factor. It turns out to be the case under some commutativity constraint on the transition matrices.

\begin{corollary} \label{cor:3}
Consider a single environment with transitions $T$. Suppose that there exists an action $a_0\in \A$ such that $T_{a_0}$ commutes with $T_a$ for all $a\in \A$. Then for any $0 < \gamma_1, \gamma_2, \gamma_3 < 1$ with $\gamma_1 \neq \gamma_2$, $(T, \gamma_1), (T, \gamma_2)$ generalize to $(T, \gamma_3)$.

\end{corollary}

\begin{remark}
The commutativity condition cannot simply be removed. Indeed, we provide in Appendix~\ref{app:A.8} an example with two actions with non-commutative transition matrices for which condition~\eqref{eq:7.2} is not satisfied.
\end{remark}


\section{Experiments}
We now present empirical validations of our claims\footnote{
Code available at the following link \url{https://github.com/lviano/Identifiability_IRL}}. In particular, we verify the identifiability requirement given by \Cref{thm:identification} in the context of randomly generated transition matrices and different gridworlds with uniform additive noise in the dynamics. 

In addition, we study a \texttt{Windy-Gridworld} and a financial model that we term \texttt{Strebulaev-Whited} both involving exogenous variables in their state spaces. In agreement with \Cref{cor:exo}, the reward function is not identifiable in these environments, highlighting the necessity of imposing milder requirements than full reward recovery. For example, in \texttt{Windy-Gridworld}, we show that by observing multiple experts acting in environments with different wind distributions, we can generalize, i.e., train an optimal expert in environments with arbitrary other wind distribution, in accordance with \Cref{thm:generalization}. On the other hand, in \texttt{Strebulaev-Whited}, given the additional information that the reward function can be represented as a linear combination of some known features, we can identify the reward, validating the condition of \Cref{thm:param}. The algorithms are described in \Cref{app:B}.





\subsection{Identifiability experiments}
\paragraph{Experiments on \texttt{Random-Matrices}}
The first experiment involves randomly generated transition matrices and reward function with $|\S| = 18, |\A|=5$. This setting matches the numerical evidence in \cite{cao2021identifiability}. Their algorithm recovers the reward function but the connection with their theoretical contribution is not highlighted. On the contrary, we have no theory practice mismatch, since we verify exactly the condition in \Cref{thm:identification}. In particular, for the $100$ random seed we tried the rank of the matrix $A$ is $2|\mathcal{S}| -1 = 35$, then invoking \Cref{thm:identification} we can conclude that the reward function is identifiable up to a constant shift. We provide a visual example of the recovered reward in \Cref{fig:random-matrices} in \Cref{appendix:experiments}.

\paragraph{Experiments on \texttt{Gridworld}} As a second example of identifiability, we consider \texttt{Gridworld}, where the state space is a squared grid with $100$ states while the action set is given by $\mathcal{A} = \{ \mathrm{up},\mathrm{down},\mathrm{left},\mathrm{right}\}$ with dynamics given by $T_{\alpha}(s^\prime|s,a) = (1 - \alpha) T_{\mathrm{det}}(s^\prime|s,a) + \alpha U(s^\prime|s,a)$  where $T_{\mathrm{det}}(s^\prime|s,a)$ represents deterministic transition dynamics where for example the action $\mathrm{right}$ leads to the state on the right with probability $1$. If an action would lead outside the grid, then the agent stays in the current state with probability $1$. The dynamics $U(s^\prime|s,a)$ are instead uniform over the states that are first adjacent to the current state. In other words, $U(\cdot|s,a) = \mathrm{Unif}(\mathcal{N}(s)) \quad \forall a\in\mathcal{A}$ where $\mathcal{N}(s)$ denotes the set of first neighbors of the state $s$.

We generate two different environments changing the value of $\alpha$, choosing $\alpha^1=0.4$ and $\alpha^2=0.2$. We notice that, even using the same discount factor $\gamma = 0.9$, the condition of \Cref{thm:identification} holds. When $\alpha$ is kept fixed, we also notice that the condition of \Cref{cor:2} holds, and hence the reward can be recovered by just varying the discount factor $\gamma$ of the experts. We numerically verify that the reward can indeed be identified up to a constant shift in these two settings (see \Cref{fig:action_gridworld}).




\begin{figure}
    \centering
    \begin{tabular}{cccccc}
\subfloat[$r_\alpha$]{%
       \includegraphics[width=0.16\linewidth]{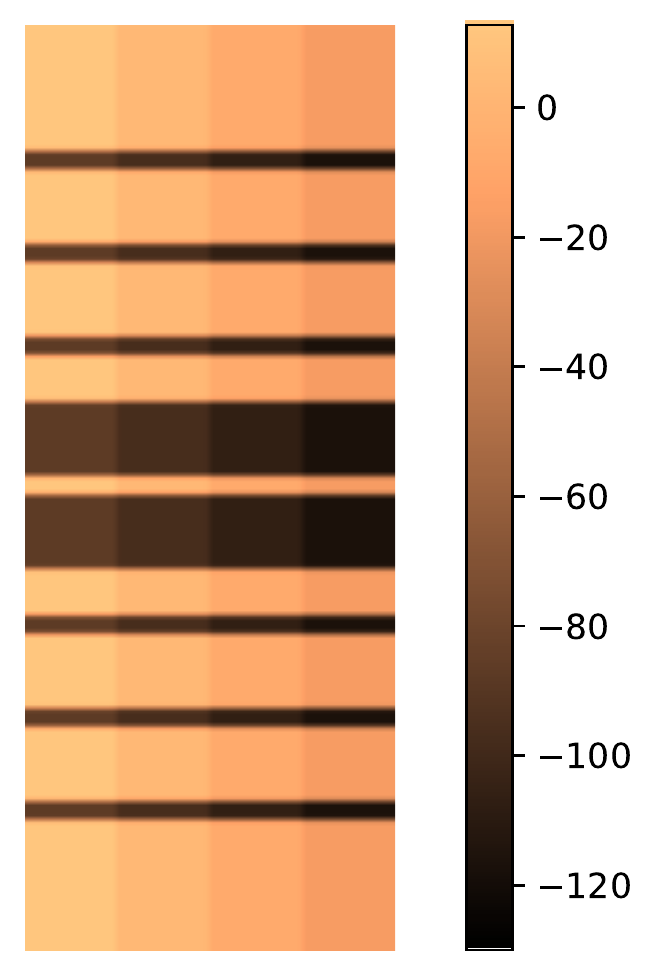}
     } & 
\subfloat[$r_\gamma$]{%
       \includegraphics[width=0.16\linewidth]{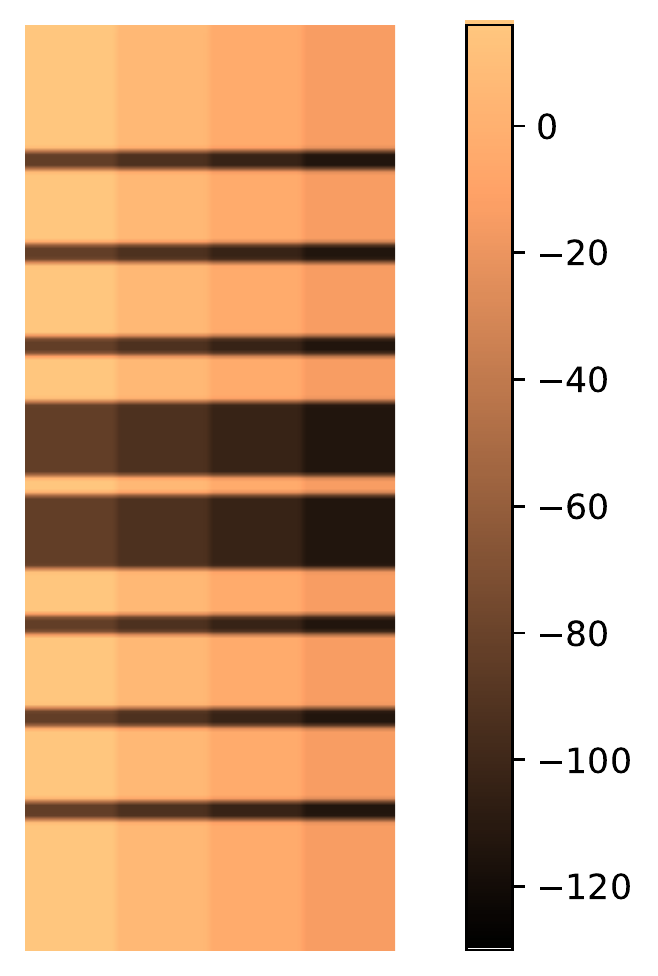}
     } &
\subfloat[$r_\mathrm{true}$]{%
       \includegraphics[width=0.16\linewidth]{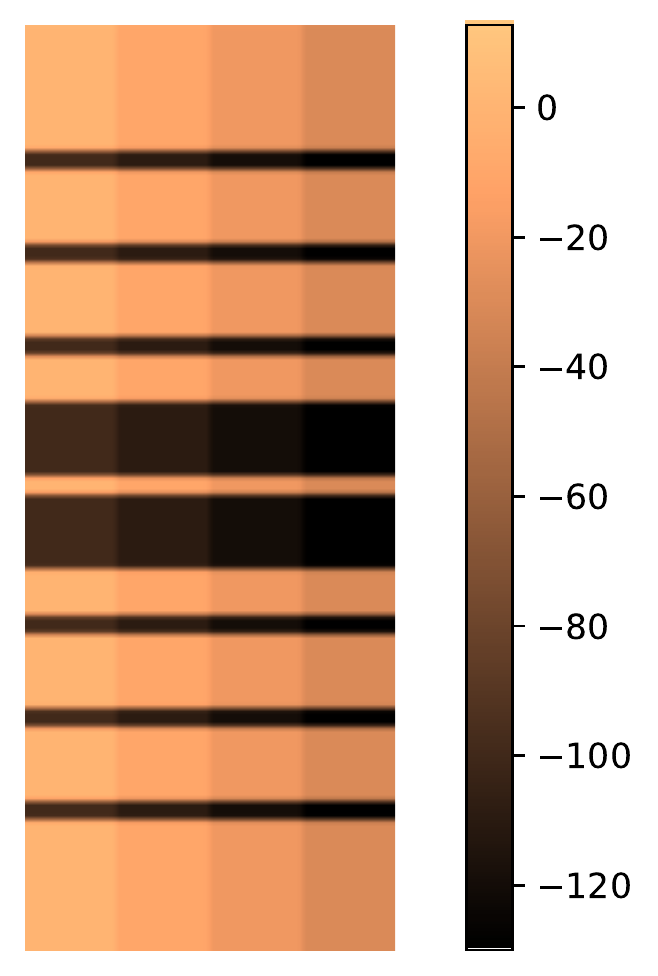}
     } &
\subfloat[$r_\alpha - r_{\mathrm{true}}$]{%
       \includegraphics[width=0.16\linewidth]{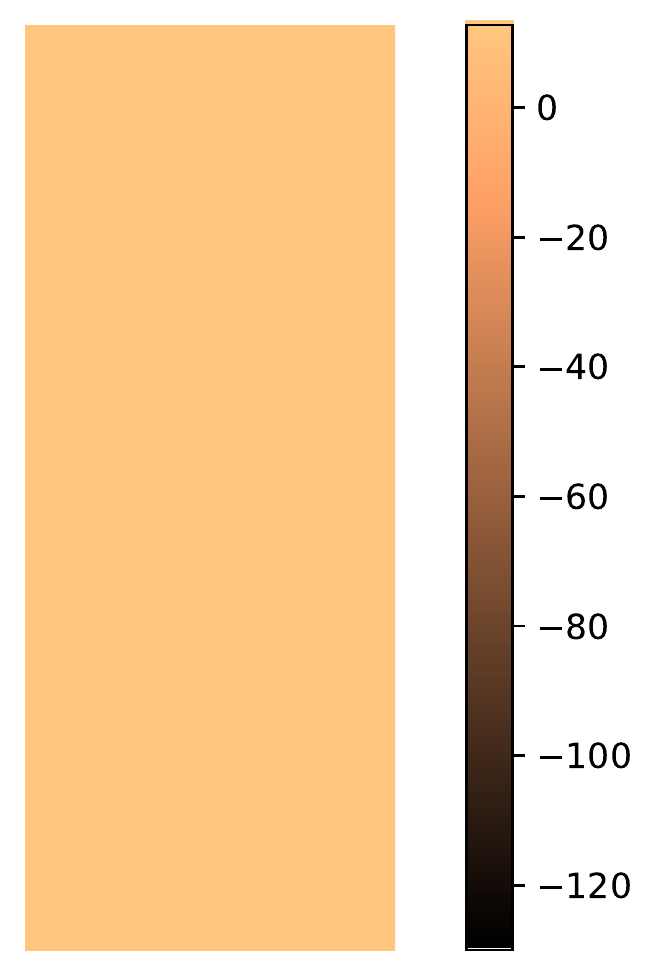}
     } &
\subfloat[$r_\gamma - r_{\mathrm{true}}$]{%
       \includegraphics[width=0.16\linewidth]{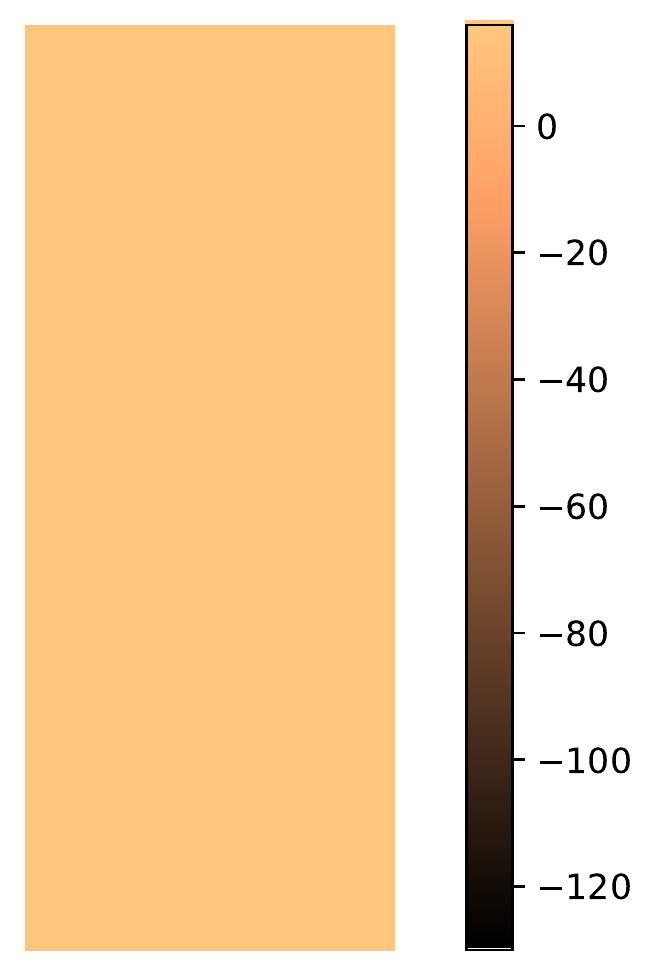}
     } \\
\end{tabular}
    \caption{Comparison between true and recovered reward in \texttt{Gridworld} with an action dependent reward, $|\mathcal{S}|=100$. It can be noticed that the reward function $r_\gamma$ recovered changing discount factors is within a constant shift from the true reward ( subplots (b),(e)). The same conclusion holds for $r_\alpha$ recovered from different $\alpha$(see subplots (a),(d)).}
\label{fig:action_gridworld}
\end{figure}
\subsection{Generalizability experiments}

In this section, we present cases where identifiability is not possible due to the presence of exogenous variables. However, we notice that the generalizability condition in \Cref{thm:generalization} is often satisfied, even for a test environment with parameters rather different than the environments of the observed experts. We start briefly describing the environments to later comment on the results.


\paragraph{Experiments on \texttt{WindyGridworld}} The \texttt{WindyGridworld} environment augments the \texttt{Gridworld} state representation by including a wind direction. The wind impacts the position transitions by making the agent move one step in the direction of the wind in addition to the action taken. The wind directions at step $t$, $w_t$ are sampled i.i.d. from the distribution $P_{\mathrm{wind}}$, and is hence an exogenous variable. While the reward is not identifiable whatever the number of experts, we can generalize to a new environment with an arbitrary wind distribution by observing enough experts in environments with different wind distributions.

\begin{wrapfigure}{r}{0.55\textwidth}
\centering
    \begin{tabular}{cc}
\subfloat[Generalizability \label{fig:gen_condition}]{%
       \includegraphics[width=0.47\linewidth]{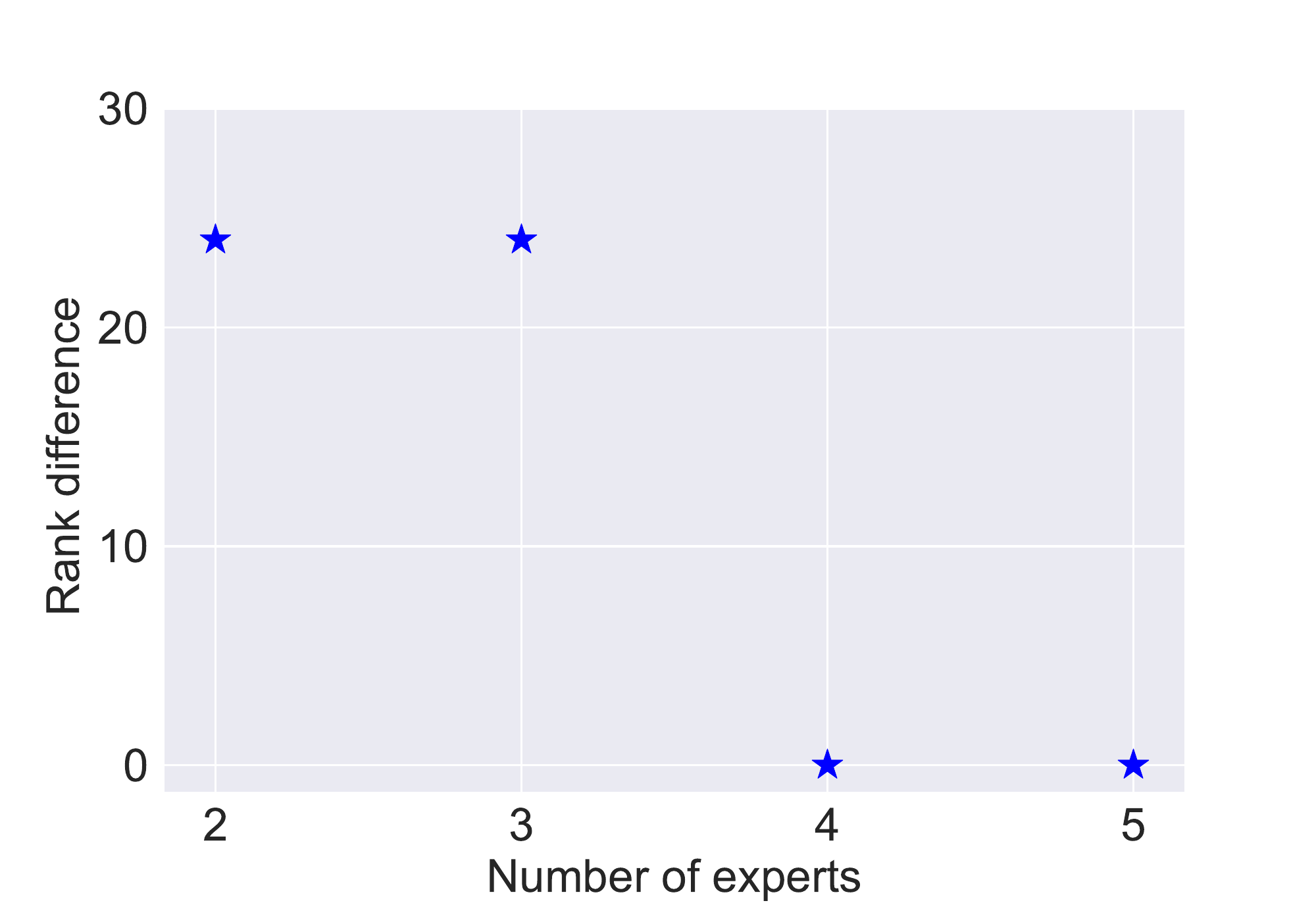}
     } &
\subfloat[ Identifiability \label{fig:id_condition}]{%
       \includegraphics[width=0.47\linewidth]{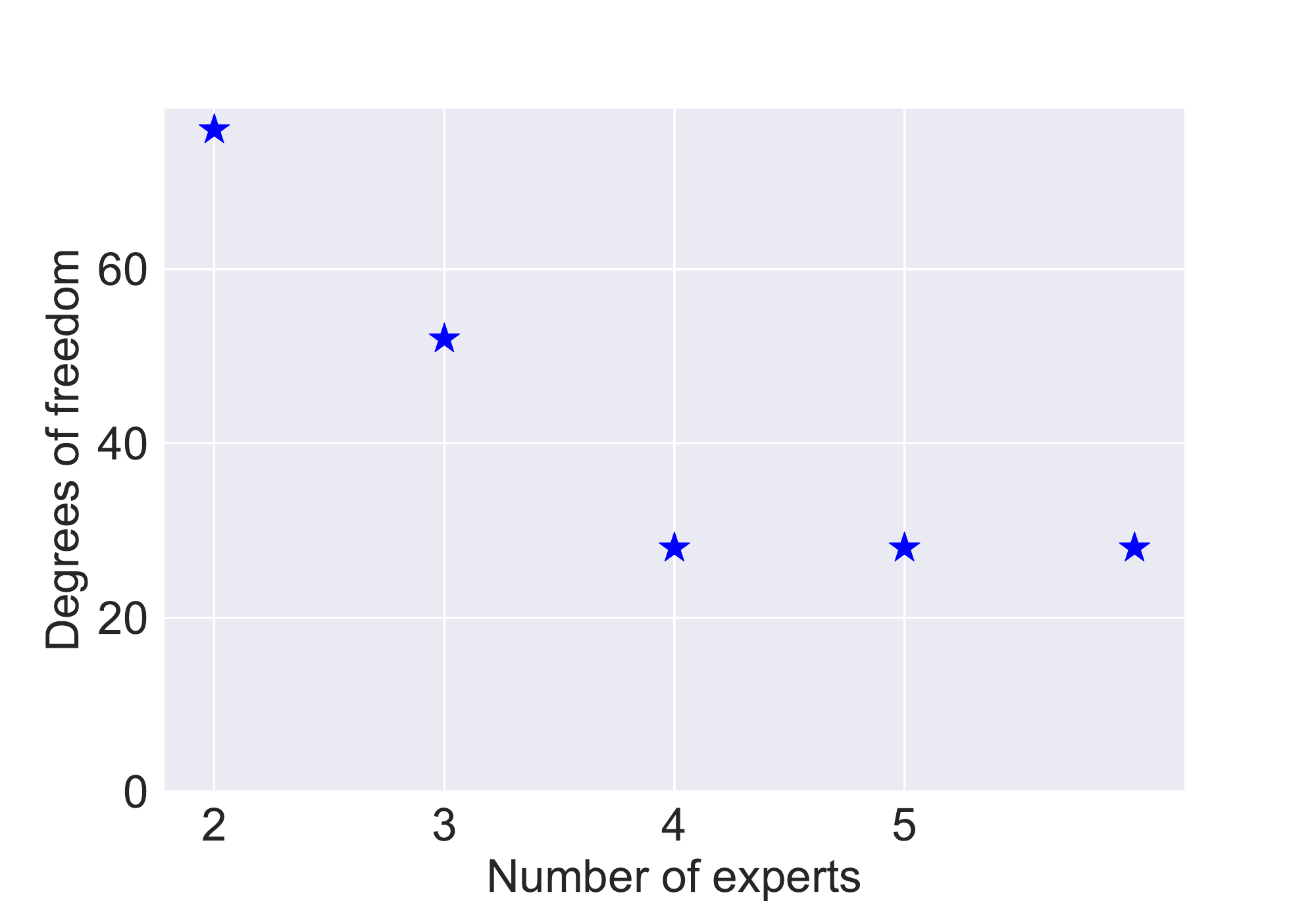}
     } \\
     \end{tabular}
     \caption{\Cref{fig:gen_condition} shows the difference between right and left term of \Cref{thm:generalization}. \Cref{fig:id_condition} shows the difference between columns and rank of the matrix in \Cref{thm:identification}.We have identifiability or generalizability respectively when those values are $0$.}
     \vspace{-3mm}
\end{wrapfigure}

In \Cref{fig:id_condition}, we see that we can obtain better identifiability (although never full identifiability) when increasing the number of experts. Once we have observed $4$ experts, we do not get further identifiability by observing more experts, hence leading to generalizability as shown in \Cref{fig:gen_condition} and \Cref{fig:windygridworld}.  We conjecture that this number of experts is linked to the number of values that the exogenous variable, i.e. the wind direction, can take.  

Furthermore, although the actions in Gridworld do not exactly commute (because of the boundary), observing two experts in the same environment with different discount factors enables generalizing to a different discount factor (see \Cref{fig:gammawindygridworld} in \Cref{appendix:experiments}).  The condition of \Cref{cor:3} is hence sufficient but not necessary.

\begin{figure}
    \centering
    \begin{tabular}{ccccc}
\subfloat[$r$]{%
       \includegraphics[width=0.16\linewidth]{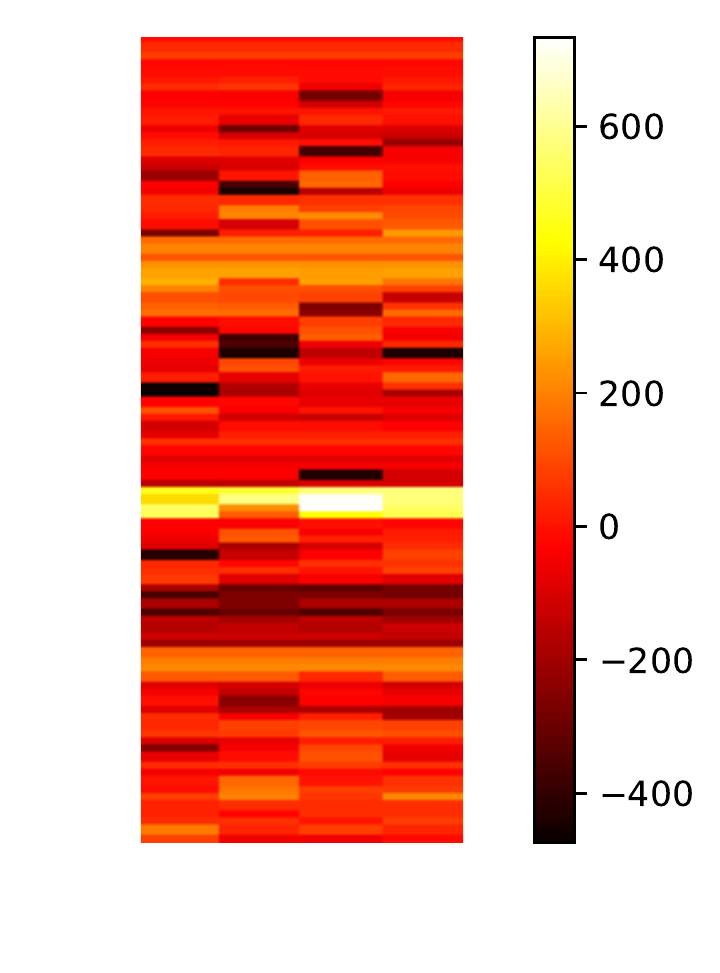}
     } &
\subfloat[$r_\mathrm{true}$]{%
       \includegraphics[width=0.16\linewidth]{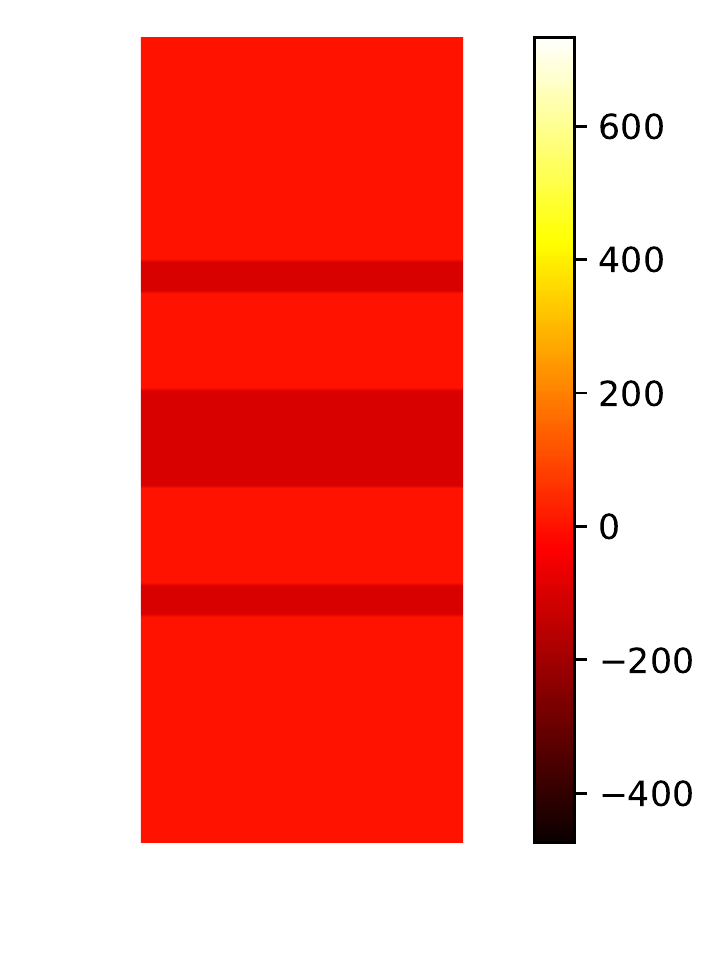}
     } &
\subfloat[$r - r_{\mathrm{true}}$]{%
       \includegraphics[width=0.16\linewidth]{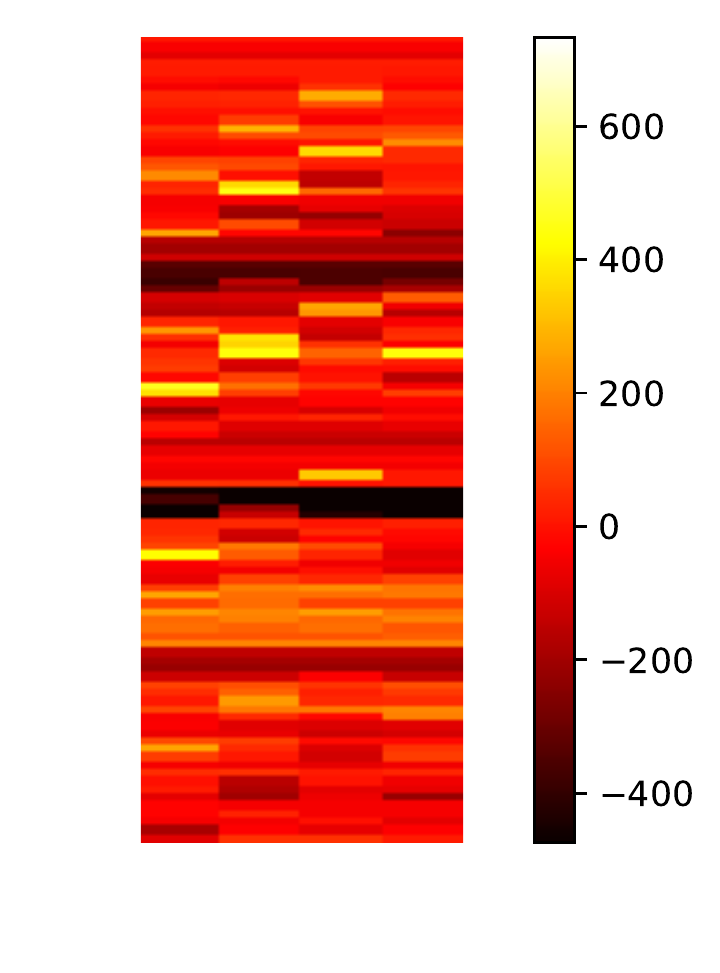}
     } &
\subfloat[$\pi^{T_{\mathrm{test}}}_{r_{\mathrm{true}}}$]{%
       \includegraphics[width=0.16\linewidth]{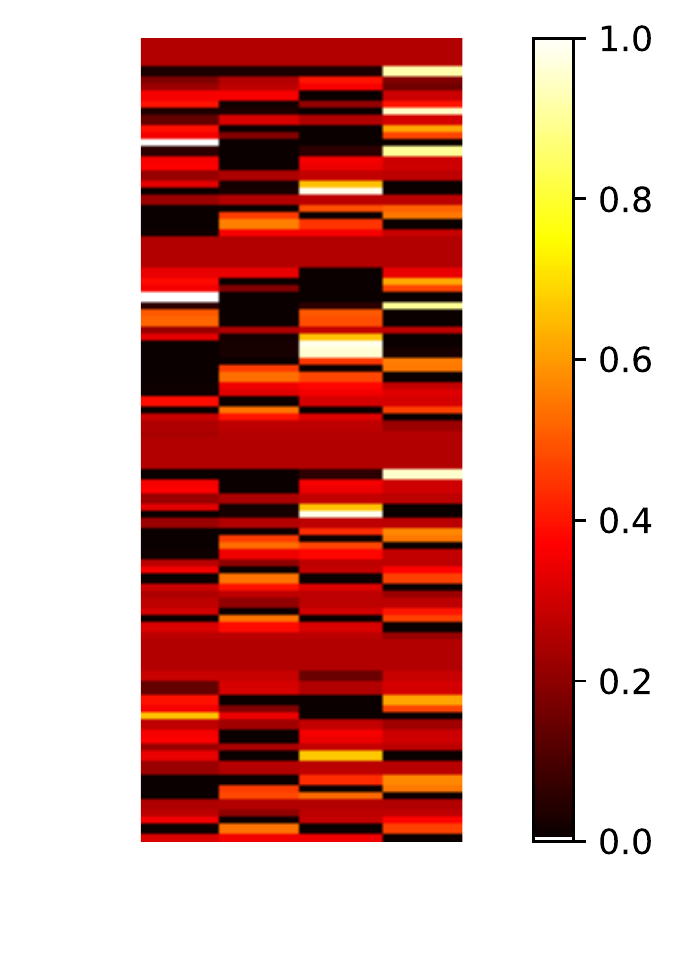}
     } &
\subfloat[$\pi^{T_{\mathrm{test}}}_{r_{\mathrm{true}}} - \pi^{T_{\mathrm{test}}}_r$]{%
       \includegraphics[width=0.16\linewidth]{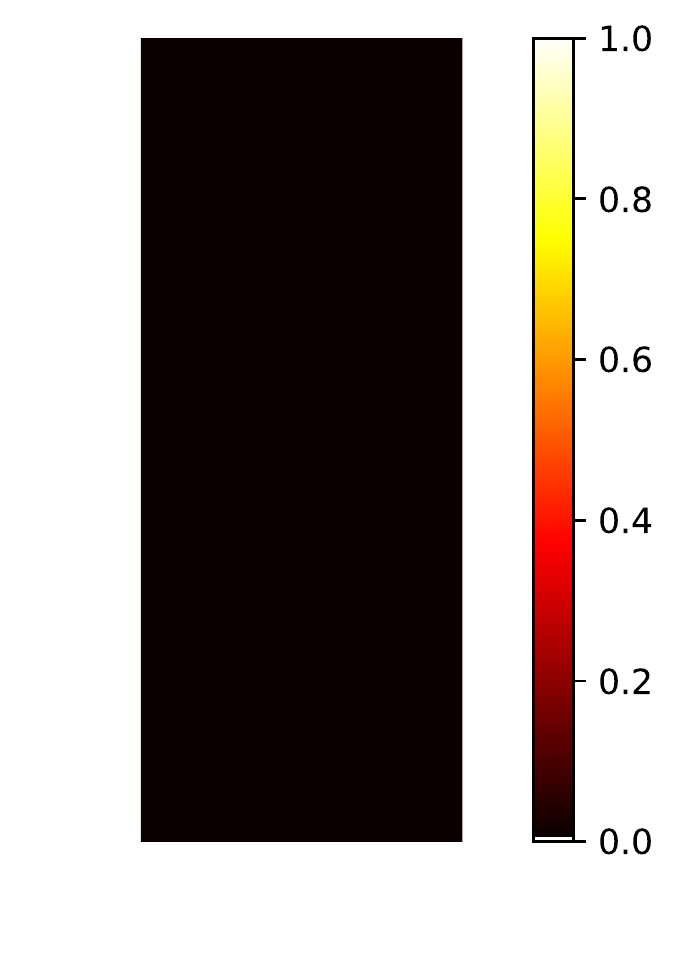}
    } \\
\end{tabular}
    \caption{Comparison between true and recovered reward ($r$ and $r_{\mathrm{true}}$) from $4$ experts in \texttt{WindyGridworld} with $|\mathcal{S}|=400$. We notice that the reward function is not identified (see (a), (b), (c)). However, when we use the recovered reward in subplot (a) to train an optimal policy under unseen dynamics we recover the optimal policy under the true reward in subplot (b). The subplot (d) shows the policy $\pi^{T_{\mathrm{test}}}_{r_{\mathrm{true}}}$ recovered from the true reward in a new environment $T_{\mathrm{test}}$ and (e) shows the difference between the policy recovered from $r_{\mathrm{true}}$ and from the recovered reward denoted as $\pi^{T_{\mathrm{test}}}_r$.}
\label{fig:windygridworld}
\vspace{-3mm}
\end{figure}

\paragraph{Experiments on \texttt{Strebulaev-Whited}}
The \texttt{Strebulaev-Whited} environment is the neoclassical investment model in which a firm has a Cobb-Douglas production function with decreasing returns to scale, as in  \cite{strebulaev2012dynamic}. The goal of the agent is to maximize profits discounted at rate $0<\gamma<1$. The state of the agent is defined by the capital level $k\geq 0$ and an exogenously given persistent stochastic productivity shock $z$. We can summarize the state by $s=(k,z)$. The next state $s^\prime=(k^\prime,z^\prime)$ is determined separately for $k^\prime$ and $z^\prime$. We have that $k'=(1-\delta)k+ak,$ where $\delta$ is the depreciation rate of physical capital and $a$ is today's rate of investment which is the action in the model. The variable $z^\prime$ evolves according to $\ln z'= \rho\ln z+\epsilon$ where $\epsilon\sim N(0,\sigma_{\epsilon})$.

The continuous variables $k$ and $z$ are discretized according to the scheme proposed in \cite{Tauchen:1986}. Hence, we obtain a discrete process with $K^2$ possible values for the state variable $s=(k,z)$ (so $|\mathcal{S}|=K^2$) and $K$ values for the action $a$. In the experiments in \Cref{fig:strebulaev} in \Cref{appendix:experiments} , we choose $K=20$ and consider two environments with different values of $\sigma_{\epsilon}$ set to $0.02$ and $0.04$, respectively. We observe that the rank of the identifiability matrix is $552$. Since $552 < 2|\mathcal{S}| -1 = 799$, the reward function is not identifiable up to a constant as expected in MDPs with exogenous states. Nonetheless, when we consider a third environment with $\sigma_{\epsilon}=0.6$, the generalizability condition in \Cref{thm:generalization} is satisfied. Hence, the expert behavior can be predicted in the third environment (see \Cref{fig:strebulaev_e} in \Cref{appendix:experiments} ).

\subsection{Identifiability experiments with a restricted reward class}
\looseness=-1
The final result presents a numerical validation of \Cref{thm:param} in the environment \texttt{Strebulaev-Whited} with exogenous state variable. In this model, the true reward function can be expressed as a linear combination of the three features given by $f_{s,a} = [z((1 - \delta)k + ak)^{\theta}, (1 - \delta)k, ak]^T$, where $s=(k,z)$ and the parameter $\rho\in(0,1)$ captures the curvature of the production function. We set $\rho={0.55}$. The first feature corresponds to the firm's output or sales which is available from the firm's income statement, the second feature is the firm's current capital stock net of depreciation which is available from the balance sheet, and the third feature is the level of investment that determines the future level of capital stock. The true reward function can be written as $r(s,a) = w^T f_{s,a}$ with $w = [1,1,-1]^T$. It can be interpreted as follows: the agent's reward of investment is an increase in output/sales, $w_1=1>0$, while the cost of capital is $1$ and, hence, investment is costly, $w_3=-1<0$. At the same time, the capital stock is valuable and can be liquidated at a price of $w_2=1>0$.

Knowing these features, we can verify that the rank of the matrix in \Cref{eq:paramR}, is $803$ which is equal to $2|S| + d$ in this environment ($|S|=400$ and $d=3$). Invoking \Cref{thm:param}, we can conclude that the reward function is identifiable exactly, which is verified numerically in \Cref{fig:lin_strebulaev} in \Cref{appendix:experiments}.
Expressing the reward in terms of features hence helps identifiability and interpretability.
\section{Conclusion}
In this paper, we analyze conditions that guarantee identifiability of the reward function (up to an additive constant) from multiple observed experts maximizing the same reward and facing different transition dynamics. This allows us to train optimal policies in any other environment sharing the same reward with the environments of the observed experts. On the other hand, in order to generalize to unknown environments, such strong reward identification is not required, and we provide a milder necessary and sufficient condition for generalizability. We also provide identifiability results in a variety of settings, i.e., linearly parameterized reward, approximated transition matrices, observation of any number of experts, as well as a non-identifiability result in the presence of exogenous variables. In the following, we list the main limitations of our work that will be the subject of future studies.

\paragraph{Observing experts in different environments.} We saw that observing a single expert in one environment cannot lead to reward identification in our setting. We hence need to observe at least two experts acting in different enough environments. To motivate this assumption, note that varying environments are ubiquitous in RL, in particular in Robust RL which deals with the training of experts that perform well in different environments, where the transition dynamics can vary to some extent. It is hence rather common to consider that the transition dynamics of a given environment can change. This was studied, e.g., in \cite{kamalaruban2020robust, pinto2017robust, tessler2019action, viano2022robust}, where the authors considered different Mujoco environments with varying friction coefficients, or object masses, which influence the dynamics. Also, instead of observing different experts in different environments, we could imagine that we observe a single expert in a single environment that varies over time (but with fixed reward), and that the expert adapts to these changes. Such observations would provide us optimal actions in environments with different transition dynamics, and thus our results would apply. This is of particular interest in economics/finance where the environment is in constant evolution.

\paragraph{Assuming entropy regularized experts.} When observing real world data, we have to face the fact that humans do not follow this idealized mathematical model. However, it turns out that our results still hold for the more general class of regularized MDPs \cite{Geist:2019} where we replace the entropy with any strongly convex function (see Appendix~\ref{app:regularized}). Whether the flexibility in the choice of the strongly convex regularizer allows to better capture real-world behaviors is an open question.

\section*{Acknowledgements}
This work has received financial support from the Enterprise for Society Center (E4S) and SNF project 100018\_192584. This work was supported by the Swiss National Science Foundation (SNSF) under  grant number 200021\_205011. This project has received funding from the European Research Council (ERC) under the European Union's Horizon 2020 research and innovation programme (grant agreement n° 725594 - time-data). 

\bibliographystyle{unsrt}
\bibliography{neurips}
\section*{Checklist}

\begin{enumerate}

\item For all authors...
\begin{enumerate}
  \item Do the main claims made in the abstract and introduction accurately reflect the paper's contributions and scope?
    \answerYes{}
  \item Did you describe the limitations of your work?
    \answerYes{}
  \item Did you discuss any potential negative societal impacts of your work?
    \answerNo{}
  \item Have you read the ethics review guidelines and ensured that your paper conforms to them?
    \answerYes{}
\end{enumerate}

\item If you are including theoretical results...
\begin{enumerate}
  \item Did you state the full set of assumptions of all theoretical results?
    \answerYes{}
        \item Did you include complete proofs of all theoretical results?
    \answerYes{}
\end{enumerate}

\item If you ran experiments...
\begin{enumerate}
  \item Did you include the code, data, and instructions needed to reproduce the main experimental results (either in the supplemental material or as a URL)?
    \answerYes{}
  \item Did you specify all the training details (e.g., data splits, hyperparameters, how they were chosen)?
    \answerYes{}
        \item Did you report error bars (e.g., with respect to the random seed after running experiments multiple times)?
    \answerNA{}
        \item Did you include the total amount of compute and the type of resources used (e.g., type of GPUs, internal cluster, or cloud provider)?
    \answerYes{See the supplementary.}
\end{enumerate}

\item If you are using existing assets (e.g., code, data, models) or curating/releasing new assets...
\begin{enumerate}
  \item If your work uses existing assets, did you cite the creators?
    \answerNA{}
  \item Did you mention the license of the assets?
    \answerNA{}
  \item Did you include any new assets either in the supplemental material or as a URL?
    \answerNA{}
  \item Did you discuss whether and how consent was obtained from people whose data you're using/curating?
    \answerNA{}
  \item Did you discuss whether the data you are using/curating contains personally identifiable information or offensive content?
    \answerNA{}
\end{enumerate}

\item If you used crowdsourcing or conducted research with human subjects...
\begin{enumerate}
  \item Did you include the full text of instructions given to participants and screenshots, if applicable?
    \answerNA{}
  \item Did you describe any potential participant risks, with links to Institutional Review Board (IRB) approvals, if applicable?
    \answerNA{}
  \item Did you include the estimated hourly wage paid to participants and the total amount spent on participant compensation?
    \answerNA{}
\end{enumerate}

\end{enumerate}

\newpage


\appendix

\section{Proofs} \label{app:A}

We provide hereafter the proofs of the statements made in the main body. 

\subsection{Proof of Theorem~\ref{thm:identification}} \label{app:A1}

Let $r \in \R^{|\S|\times |\A|}$ be the reward function common to the two experts, and let $v^1, v^2 \in \R^{|\S|}$ be the entropy regularized values functions associated experts 1 and 2 respectively and the reward function $r$. Then, according to Theorem~\ref{thm:single_expert}, we have that $\forall (s,a)\in\S\times\A$,
\begin{align*}
    r(s,a) &= \lambda \log \pi^1(a|s) - \gamma_1 \sum_{s' \in \S} T^1(s'|s,a) v^1(s') + v^1(s) \\
    &= \lambda \log \pi^2(a|s) - \gamma_2 \sum_{s' \in \S} T^2(s'|s,a) v^2(s') + v^2(s)
\end{align*}
where $\pi^1, \pi^2$ denote the policies of experts 1 and 2 respectively. We hence deduce that $\forall a\in\A$,
\begin{equation}
    \begin{pmatrix}
    I - \gamma_1 T_a^1 & -(I - \gamma_2 T_a^2)
    \end{pmatrix}
    \begin{pmatrix}
    v^1 \\
    v^2
    \end{pmatrix}
    = 
    \lambda \log \pi^2(a|\cdot) - \lambda \log \pi^1(a|\cdot).
\end{equation}

By including all available actions to the experts, we can write
\begin{equation} \label{eq:A1_appendix}
    \begin{pmatrix}
I - \gamma_1 T_{a_1}^1 & -(I - \gamma_2 T_{a_1}^2) \\
\vdots & \vdots \\
I - \gamma_1 T_{a_{|\mathcal{A}}|}^1 & -(I - \gamma_2 T_{a_{|\mathcal{A}|}}^2)
\end{pmatrix}
\begin{pmatrix}
    v^1 \\
    v^2
    \end{pmatrix}
    =
    \begin{pmatrix}
    \lambda \log \pi^2(\cdot|a_1) - \lambda \log \pi^1(\cdot|a_1) \\
    \vdots \\
    \lambda \log \pi^2(\cdot|a_{|\A|}) - \lambda \log \pi^1(\cdot|a_{|\A|})
    \end{pmatrix}.
\end{equation}

Reward identifiability is directly related to the size of the solution space of the linear system~\eqref{eq:A1_appendix}. Since we assume that both experts are optimal with respect to a \textit{true} reward function $r$, we know that the associated value function solves equation~\eqref{eq:A1_appendix}, and hence that this system is feasible. The solution space then depends on the rank of the matrix in the left hand side of ~\eqref{eq:A1_appendix}, which we denote by $A$.

We first show that there always exists an eigenvector of $A$ associated with eigenvalue $0$. Indeed, since the matrices $T_a$ are transition matrices, their rows must sum to $1$, which can be written as $T_a \mathbf{1} = \mathbf{1}$ where $\mathbf{1}$ is a $\S$ dimensional column vector of $1$'s. Hence,

\begin{align*}
    A
    \begin{pmatrix}
    \frac{1}{1-\gamma_1} \mathbf{1} \\
    \frac{1}{1-\gamma_2} \mathbf{1}
    \end{pmatrix}
    = 
    \begin{pmatrix}
    \frac{1}{1-\gamma_1} (\mathbf{1} - \gamma_1 T^1_{a_1} \mathbf{1}) - \frac{1}{1-\gamma_2} (\mathbf{1} - \gamma_2 T^2_{a_1} \mathbf{1}) \\
    \vdots \\
    \frac{1}{1-\gamma_1} (\mathbf{1} - \gamma_1 T^1_{a_{|\A|}} \mathbf{1}) - \frac{1}{1-\gamma_2} (\mathbf{1} - \gamma_2 T^2_{a_{|\A|}} \mathbf{1})
    \end{pmatrix}
    = \textbf{0}
\end{align*}

Hence, the vector $\begin{pmatrix}
    \frac{1}{1-\gamma_1} \mathbf{1} \\
    \frac{1}{1-\gamma_2} \mathbf{1}
    \end{pmatrix}$ is an eigenvector of $A$ with eigenvalue $0$, and corresponds to the invariance of the optimal policy under addition of a constant to the reward function.
    
Suppose now that $\text{rank}(A) = 2|\S| - 1$. Since $A$ has $2\S$ columns, this implies that the only eigenvector with eigenvalue $0$ is $\begin{pmatrix}
    \frac{1}{1-\gamma_1} \mathbf{1} \\
    \frac{1}{1-\gamma_2} \mathbf{1}
    \end{pmatrix}$, and thus that we can recover the value function $v^1$ (or $v^2$ equivalently) up to an additive constant. Using Theorem~\ref{thm:single_expert} again, it implies that we can also recover the reward function up to a constant.
    
On the other hand, suppose that $\text{rank}(A) < 2|\S| - 1$. Then, there exists another vector in $\text{Ker}(A)$ which is linearly independent of $\begin{pmatrix}
    \frac{1}{1-\gamma_1} \mathbf{1} \\
    \frac{1}{1-\gamma_2} \mathbf{1}
    \end{pmatrix}$, and whose addition to the value function would not change the optimal policy. However, it is easy to check that the only eigenvector of $A$ with eigenvalue $0$ of the form $\begin{pmatrix}
    c_1 \mathbf{1} \\
    c_2 \mathbf{1}
    \end{pmatrix}$ with $c_1, c_2 \in \R$ is proportional to $\begin{pmatrix}
    \frac{1}{1-\gamma_1} \mathbf{1} \\
    \frac{1}{1-\gamma_2} \mathbf{1}
    \end{pmatrix}$. Hence, any other vector in $\text{Ker}(A)$ would induce a modification of the value and reward functions more complex than just adding a constant. The provided condition is hence also necessary.

\paragraph{Equivalence with Definition~\ref{def:distinguishability}.} It turns out that our rank condition~\eqref{eq:7.1} is equivalent to the value-distinguishing assumption of Definition~\ref{def:distinguishability}. To show this, we first notice that, if $v^1, v^2$ satisfy equation~\eqref{eq:disting}, and if $v^1$ is a constant vector, then $v^2$ must also be a constant vector, and vice versa. Indeed, equation~\eqref{eq:disting} can be written as
\[
(I - \gamma_1 T^1_a) v^1 = (I - \gamma_2 T^2_a) v^2 \ \forall a\in \A.
\]
Since, $\forall a\in\A, i=1,2$, $\mathbf{1}$ is an eigenvector of $T_a^i$ with eigenvalue 1, then $\mathbf{1}$ is also an eigenvector of $(I-\gamma_2 T_a^2)^{-1}$ with eigenvalue $\frac{1}{1-\gamma_2}$. Hence, if $v^1 = c\mathbf{1}$ is a constant vector, then $v^2 = (I-\gamma_2 T_a^2)^{-1} (I - \gamma_1 T^1_a) v^1 = c\frac{1-\gamma_1}{1-\gamma_2} \mathbf{1}$ is also a constant vector, and the associated constant is determined by the constant of $v^1$. Thus, the condition of Definition~\ref{def:distinguishability} can be rewritten as
\[
(I - \gamma_1 T^1_a) v^1 = (I - \gamma_2 T^2_a) v^2 \ \forall a\in \A \Rightarrow (v^1, v^2) = (c\mathbf{1}, c\frac{1-\gamma_1}{1-\gamma_2} \mathbf{1}) \text{ for some } c\in \R.
\]
This is hence equivalent to
\[
\text{dim} \left(\text{Ker}
\begin{pmatrix}
I - \gamma_1 T_{a_1}^1 & I - \gamma_2 T_{a_1}^2 \\
\vdots & \vdots \\
I - \gamma_1 T_{a_{|\mathcal{A}}|}^1 & I - \gamma_2 T_{a_{|\mathcal{A}|}}^2
\end{pmatrix} \right)
= 1.
\]
which is equivalent to equation~\eqref{eq:7.1}.

\subsection{Proof of Corollary~\ref{cor:3experts}}

Let $v^1, v^2, v^3 \in \R^{|\S|}$ be the entropy regularized value functions associated with experts $1, 2$ and $3$ respectively. Following the proof of Theorem~\ref{thm:identification}, these vectors must satisfy

\begin{equation} \label{eq:3experts.1}
    \begin{pmatrix}
I - \gamma_1 T_{a_1}^1 & -(I - \gamma_2 T_{a_1}^2) & \textbf{0} \\
\vdots & \vdots & \vdots \\
I - \gamma_1 T_{a_{|\mathcal{A}|}}^1 & -(I - \gamma_2 T_{a_{|\mathcal{A}|}}^2) & \textbf{0} \\
I - \gamma_1 T_{a_1}^1 & \textbf{0} & -(I - \gamma_3 T_{a_1}^3) \\
\vdots & \vdots & \vdots \\
I - \gamma_1 T_{a_{|\mathcal{A}|}}^1 & \textbf{0} & -(I - \gamma_3 T_{a_{|\mathcal{A}|}}^3)
\end{pmatrix}
\begin{pmatrix}
    v^1 \\
    v^2 \\
    v^3
    \end{pmatrix}
    =
    \begin{pmatrix}
    \lambda \log \pi^2(\cdot|a_1) - \lambda \log \pi^1(\cdot|a_1) \\
    \vdots \\
    \lambda \log \pi^2(\cdot|a_{|\A|}) - \lambda \log \pi^1(\cdot|a_{|\A|}) \\
    \lambda \log \pi^3(\cdot|a_1) - \lambda \log \pi^1(\cdot|a_1) \\
    \vdots \\
    \lambda \log \pi^3(\cdot|a_{|\A|}) - \lambda \log \pi^1(\cdot|a_{|\A|})
    \end{pmatrix}.
\end{equation}

Similarly as previously, we can easily show that the vector $\begin{pmatrix}
    \frac{1}{1-\gamma_1} \mathbf{1} \\
    \frac{1}{1-\gamma_2} \mathbf{1} \\
    \frac{1}{1-\gamma_3} \mathbf{1}
    \end{pmatrix} \in \text{Ker}(A')$, where $A'$ denotes the matrix on the left of equation~\eqref{eq:3experts.1} In order for the reward to be recovered up to a constant, we hence need that there is no other linearly independent vector in $\text{Ker}(A')$, i.e., that $\text{rank}(A') = 3|\S| - 1$.


\subsection{Proof of Corollary~\ref{cor:2}}

We want to show that

\begin{equation}\label{eq:Cor2.0}
\text{dim}\left(\text{Ker}
\begin{pmatrix}
I - \gamma_1 T_{a_1} & -(I - \gamma_2 T_{a_1}) \\
\vdots & \vdots \\
I - \gamma_1 T_{a_{|\mathcal{A}}|} & -(I - \gamma_2 T_{a_{|\mathcal{A}|}})
\end{pmatrix} \right)
= 1.
\end{equation}

Suppose that $\begin{pmatrix} v^1 \\ v^2 \end{pmatrix} \in \text{Ker}
\begin{pmatrix}
I - \gamma_1 T_{a_1} & -(I - \gamma_2 T_{a_1}) \\
\vdots & \vdots \\
I - \gamma_1 T_{a_{|\mathcal{A}}|} & -(I - \gamma_2 T_{a_{|\mathcal{A}|}})
\end{pmatrix}$, i.e.,

\begin{equation} \label{eq:Cor2.1}
(I - \gamma_1 T_a)v^1 = (I - \gamma_2 T_a)v^2 \ \forall a\in \A,
\end{equation}
or equivalently
\begin{equation} \label{eq:Cor2.2}
v^1 - v^2 = T_a(\gamma_1 v^1 - \gamma_2v^2) \ \forall a\in \A.
\end{equation}
Subtracting equation~\eqref{eq:Cor2.2} for $a=a_1$ and $a=a_i$, we get
\begin{equation} \label{eq:Cor2.3}
(T_{a_1} - T_{a_i})(\gamma_1 v^1 - \gamma_2v^2) = 0 \ \forall i.
\end{equation}

Using equation~\eqref{eq:7.1.2} and the fact that the vector $\bf{1} \in \R^{|\S|}$ always belongs to $\text{Ker}
    \begin{pmatrix}
    T_{a_1} - T_{a_2} \\
    \vdots \\
    T_{a_1} - T_{a_{|\mathcal{A}|}}
    \end{pmatrix}$, we have that $\text{Ker}
    \begin{pmatrix}
    T_{a_1} - T_{a_2} \\
    \vdots \\
    T_{a_1} - T_{a_{|\mathcal{A}|}}
    \end{pmatrix} = \text{Span}(\bf{1})$. Thus, we deduce from equation~\eqref{eq:Cor2.3} that 
    
\begin{equation} \label{eq:Cor2.4}
\gamma_1 v^1 - \gamma_2v^2 = c\bf{1}
\end{equation} 
for some $c\in \R$. Using the fact that for any $a\in \A$, $\bf{1}$ is an eigenvector of $T_a$ with eigenvalue $1$, we deduce from~\eqref{eq:Cor2.2} and~\eqref{eq:Cor2.4} that

\begin{equation} \label{eq:Cor2.5}
v^1 - v^2 =  T_a c\textbf{1} = c\textbf{1}.
\end{equation}
Solving equations~\eqref{eq:Cor2.4} and~\eqref{eq:Cor2.5} for $v^1$ and $v^2$, we find $v^1 = \frac{c(1-\gamma_2)}{\gamma_1 - \gamma_2}\bf{1}$ and $v^2 = \frac{c(1-\gamma_1)}{\gamma_2 - \gamma_1}\bf{1}$. Therefore,

\[
\text{Ker}
\begin{pmatrix}
I - \gamma_1 T_{a_1}^1 & -(I - \gamma_2 T_{a_1}^2) \\
\vdots & \vdots \\
I - \gamma_1 T_{a_{|\mathcal{A}}|}^1 & -(I - \gamma_2 T_{a_{|\mathcal{A}|}}^2)
\end{pmatrix} = \left\{\begin{pmatrix} v^1 \\ v^2 \end{pmatrix}: v^1 = \frac{c(1-\gamma_2)}{\gamma_1 - \gamma_2}\textbf{1}, v^2 = \frac{c(1-\gamma_1)}{\gamma_2 - \gamma_1}\textbf{1} \text{ for } c \in \R\right\}
\]

which shows that condition~\eqref{eq:Cor2.0} holds. On the other hand, if condition~\eqref{eq:7.1.2} does not hold, then $\text{Ker}
    \begin{pmatrix}
    T_{a_1} - T_{a_2} \\
    \vdots \\
    T_{a_1} - T_{a_{|\mathcal{A}|}}
    \end{pmatrix}$ contains another vector $v_0$ which is not a constant vector, so the reward cannot be recovered up to a constant.

\subsection{Proof of Theorem~\ref{thm:param}}

Suppose that $\mathbf{1} \in \text{Im} \begin{pmatrix} f_{a_1} \\ \vdots \\ f_{a_{|\A|}} \end{pmatrix}$, i.e., $\exists w\in\R^d$ such that $\begin{pmatrix} f_{a_1} \\ \vdots \\ f_{a_{|\A|}} \end{pmatrix} w = \mathbf{1}$. This implies that

\begin{equation}\label{eq:param1}
\begin{pmatrix}
I - \gamma_1 T_{a_1}^1 & I - \gamma_2 T_{a_1}^2 & \textbf{0} \\
\vdots & \vdots & \vdots \\
I - \gamma_1 T_{a_{|\mathcal{A}|}}^1 & I - \gamma_2 T_{a_{|\mathcal{A}|}}^2 & \textbf{0} \\
I - \gamma_1 T_{a_1}^1 & \textbf{0} & f_{a_1} \\
\vdots & \vdots & \vdots \\
I - \gamma_1 T_{a_{|\mathcal{A}|}}^1 & \textbf{0} & f_{a_{|\A|}}
\end{pmatrix} \begin{pmatrix} \frac{1}{1-\gamma_1}\mathbf{1} \\ -\frac{1}{1-\gamma_2}\mathbf{1} \\ -w \end{pmatrix}
= \mathbf{0}.
\end{equation}

Suppose that condition~\eqref{eq:paramR} holds, i.e., that

\begin{equation}\label{eq:param2}
\text{dim}\left(\text{Ker}
\begin{pmatrix}
I - \gamma_1 T_{a_1}^1 & I - \gamma_2 T_{a_1}^2 & \textbf{0} \\
\vdots & \vdots & \vdots \\
I - \gamma_1 T_{a_{|\mathcal{A}|}}^1 & I - \gamma_2 T_{a_{|\mathcal{A}|}}^2 & \textbf{0} \\
I - \gamma_1 T_{a_1}^1 & \textbf{0} & f_{a_1} \\
\vdots & \vdots & \vdots \\
I - \gamma_1 T_{a_{|\mathcal{A}|}}^1 & \textbf{0} & f_{a_{|\A|}}
\end{pmatrix} \right)
= 1.
\end{equation}

Equations~\eqref{eq:param1} and~\eqref{eq:param2} thus imply that 

\begin{equation}
\text{Ker}
\begin{pmatrix}
I - \gamma_1 T_{a_1}^1 & I - \gamma_2 T_{a_1}^2 & \textbf{0} \\
\vdots & \vdots & \vdots \\
I - \gamma_1 T_{a_{|\mathcal{A}|}}^1 & I - \gamma_2 T_{a_{|\mathcal{A}|}}^2 & \textbf{0} \\
I - \gamma_1 T_{a_1}^1 & \textbf{0} & f_{a_1} \\
\vdots & \vdots & \vdots \\
I - \gamma_1 T_{a_{|\mathcal{A}|}}^1 & \textbf{0} & f_{a_{|\A|}}
\end{pmatrix}
= \text{Span}\begin{pmatrix} \frac{1}{1-\gamma_1}\mathbf{1} \\ -\frac{1}{1-\gamma_2}\mathbf{1} \\ -w \end{pmatrix}.
\end{equation}

This means that for any $v^1, v^2$ satisfying $(I - \gamma_1T_a^1) v^1 = (I - \gamma_1T_a^2) v^2 \ \forall a\in \A$ and such that $\exists w \in \R^d$, $(I - \gamma_1T_a^1)v^1 =  f_a w \ \forall a\in \A$, then $v^1 \propto \mathbf{1}$.

Now suppose that we recover a reward function $r(s,a) = w^T f_{s,a}$ compatible with the two experts, i.e.,
\begin{equation} \label{eq:param3}
    r(\cdot, a) = r^*(\cdot, a) + (I - \gamma_1T_a^1) v^1
\end{equation}

where $r^*(s,a) = w^{*T} f_{s,a}$ denotes the true reward and $v^1$ satisfies $(I - \gamma_1T_a^1) v^1 = (I - \gamma_2T_a^2) v^2 \ \forall a\in \A$. Then, $(I - \gamma_1T_a^1) v^1 = r(\cdot, a) - r^*(\cdot, a) = f_a (w - w^*)$, and hence $\exists \tilde{w}\in \R^d$ such that $(I - \gamma_1T_a^1)v^1 =  f_a \tilde{w} \ \forall a\in \A$. Thus, $v^1 \propto \mathbf{1}$ and the reward is recovered up to a constant.

Suppose now that $\mathbf{1} \notin \text{Im} \begin{pmatrix} f_{a_1} \\ \vdots \\ f_{a_{|\A|}} \end{pmatrix}$. Then, the condition

\begin{equation}
\text{rank}
\begin{pmatrix}
I - \gamma_1 T_{a_1}^1 & I - \gamma_2 T_{a_1}^2 & \textbf{0} \\
\vdots & \vdots & \vdots \\
I - \gamma_1 T_{a_{|\mathcal{A}|}}^1 & I - \gamma_2 T_{a_{|\mathcal{A}|}}^2 & \textbf{0} \\
I - \gamma_1 T_{a_1}^1 & \textbf{0} & f_{a_1} \\
\vdots & \vdots & \vdots \\
I - \gamma_1 T_{a_{|\mathcal{A}|}}^1 & \textbf{0} & f_{a_{|\A|}}
\end{pmatrix}
= 2|\mathcal{S}| + d.
\end{equation}

means that this matrix is full rank and hence that its kernel is $\{\mathbf{0}\}$. Thus, if we recover a reward of the form~\eqref{eq:param3}, following the same argument as previously, this means that $v^1 = 0$, and thus that the reward function is recovered exactly.

\subsection{Proof of Corollary~\ref{cor:exo}}

Without loss of generality, let us assume that the exogenous variable can only take two possible values, i.e., the state space is defined as $\S = \{(s, e) : s\in \S_0, e\in\{e_1, e_2\}\}$, where $e$ denotes the exogenous variable and $\S_0$ contains all other variables. Exogenity of variable $e$ implies that $\forall e\in\{e_1, e_2\}$, $p(e^{t+1} = e_1 | s^t=s, e^t=e, a^t=a) = p(e^{t+1} = e_1 | e^t=e)$ does not depend on $s$ nor $a$.

Suppose that we order the states as $\{\{(e_1, s)\}_{s\in \S_0}, \{(e_2, s)\}_{s\in \S_0}\}$. Then, the transition matrix for each expert $i$ associated with action $a$ has the following form:

\begin{equation}
    T_a^i = \begin{pmatrix}
    p_1^i T_{a,1}^i & (1-p_1^i) T_{a,1}^i \\
    (1-p_2^i) T_{a,2}^i & p_2^i T_{a,2}^i
    \end{pmatrix}
\end{equation}

where for each expert $i$ and exogenous variable $e_j, j=1,2$ , $p_j^i = p^i(e^{t+1} = e_j | e^t=e_j)$ and $T_{a,j}^i \in \R^{|\S_0|\times |\S_0|}$ denotes the transition matrix of expert $i$ for state variables in $\S_0$ knowing that the current value of state variable $e$ is $e_j$, i.e. $T_{a,j}^i(s,s') = p^i(s^{t+1} = s' | s^t=s, e^t=e_j, a^t=a) \forall s,s' \in \S_0$ where $p^i$ denotes the state transition probability in environment $i$.

We first show the result in the case of two experts. The matrix $A = \begin{pmatrix}
I - \gamma_1 T_{a_1}^1 & I - \gamma_2 T_{a_1}^2 \\
\vdots & \vdots \\
I - \gamma_1 T_{a_{|\mathcal{A}}|}^1 & I - \gamma_2 T_{a_{|\mathcal{A}|}}^2
\end{pmatrix}$ has the following form:

\[
A = \begin{pmatrix}
    I - \gamma_1 p_1^1 T_{a_1, 1}^1 & -\gamma_1 (1-p_1^1) T_{a_1,1}^1 & I - \gamma_2 p_1^2 T_{a_1, 1}^2 & -\gamma_2 (1-p_1^2) T_{a_1,1}^2  \\
    -\gamma_1 (1-p_2^1) T_{a_1, 2}^1 & I-\gamma_1 p_2^1 T_{a_1,2}^1 & -\gamma_2 (1-p_2^2) T_{a_1, 2}^2 & I-\gamma_2 p_2^2 T_{a_1,2}^2 \\
    \vdots & \vdots & \vdots & \vdots \\
    I - \gamma_1 p_1^1 T_{a_{|\A|}, 1}^1 & -\gamma_1 (1-p_1^1) T_{a_{|\A|},1}^1 & I - \gamma_2 p_1^2 T_{a_{|\A|}, 1}^2 & -\gamma_2 (1-p_1^2) T_{a_{|\A|},1}^2  \\
    -\gamma_1 (1-p_2^1) T_{a_{|\A|}, 2}^1 & I-\gamma_1 p_2^1 T_{a_{|\A|},2}^1 & -\gamma_2 (1-p_2^2) T_{a_{|\A|}, 2}^2 & I-\gamma_2 p_2^2 T_{a_{|\A|},2}^2
    \end{pmatrix}
\]

We know that $v_0 = \begin{pmatrix}
\frac{1}{1-\gamma_1} \mathbf{1} \\
\frac{1}{1-\gamma_1} \mathbf{1} \\
-\frac{1}{1-\gamma_2} \mathbf{1} \\
-\frac{1}{1-\gamma_2} \mathbf{1} 
\end{pmatrix}$ is an eigenvector of $A$ with eigenvalue $0$, corresponding to an addition of a constant to the reward. In order to show that the reward is not identifiable, we need to find another vector in $\text{Ker}(A)$ linearly independent of $v_0$. We search for such a vector of the form $v_1 = \begin{pmatrix}
\mathbf{0} \\
\mathbf{1} \\
c_1 \mathbf{1} \\
c_2 \mathbf{1} \\
\end{pmatrix}$. Using the fact that $\mathbf{1}$ is an eigenvector of any transition matrix with eigenvalue $1$, the condition $v_1\in \text{Ker}(A)$ is equivalent to

\begin{equation*}
    \left\{ \begin{array}{cc}
         -\gamma_1(1-p_1^1) + c_1(1-\gamma_2p_1^2) - c_2\gamma_2(1-p_1^2) = 0 \\
         1-\gamma_1 p_2^1 - c_1\gamma_2(1-p_2^2) + c_2(1-\gamma_2p_2^2) = 0.
    \end{array} \right.
\end{equation*}

This system of equations turns out to have a unique solution for $(c_1, c_2)$ since 

\begin{align*}
\text{det}\begin{pmatrix}
1-\gamma_2p_1^2 & -\gamma_2(1-p_1^2) \\
-\gamma_2(1-p_2^2) & 1-\gamma_2p_2^2
\end{pmatrix} &= (1-\gamma_2p_1^2)(1-\gamma_2p_2^2) - (\gamma_2-\gamma_2p_1^2)(\gamma_2-\gamma_2p_2^2) \\
&= (1-\gamma_2)(1+\gamma_2-\gamma_2 p_1^2 - \gamma_2 p_2^2)> 0
\end{align*}
since $0 \leq \gamma_2 < 1$. Hence, $\text{Ker}(A)$ contains at least two linearly independent vector, and thus $\text{rank}(A) < 2|\S| - 1$. So, according to Theorem~\ref{thm:identification}, the reward function is not identifiable up to a constant.

This means that, in addition to a global constant that we can add to the reward, we can also add a constant only to the rewards associated with a specific value of the exogenous variable. The proof naturally extends to the case of multiple experts, and when the exogenous variable can take more than two values. Actually, in the latter case, we can find even more linearly independent vectors in $\text{Ker}(A)$, corresponding to adding a constant to the rewards associated with each possible value of the exogenous variable.

\subsection{Proof of Theorem~\ref{thm:prob}}

Define $A = \begin{pmatrix}
I - \gamma_1 T_{a_1}^1 & I - \gamma_2 T_{a_1}^2 \\
\vdots & \vdots \\
I - \gamma_1 T_{a_{|\mathcal{A}}|}^1 & I - \gamma_2 T_{a_{|\mathcal{A}|}}^2
\end{pmatrix}$ and $\hat{A} = \begin{pmatrix}
I - \gamma_1 \hat{T}_{a_1}^1 & I - \gamma_2 \hat{T}_{a_1}^2 \\
\vdots & \vdots \\
I - \gamma_1 \hat{T}_{a_{|\mathcal{A}}|}^1 & I - \gamma_2 \hat{T}_{a_{|\mathcal{A}|}}^2
\end{pmatrix}$. For an arbitrary matrix $M$, let $\sigma_2(M)$ denote the second smallest singular value of $M$. Note that the condition~\eqref{eq:7.1} for $A$ is equivalent to $\sigma_2(A) > 0$. From Weyl's inequality for singular values\cite{tao2010254a}, we have that
\[
|\sigma_2(A) - \sigma_2(\hat{A})| \leq \|A - \hat{A}\|_2.
\]

Moreover,

\begin{align*}
    \|A - \hat{A}\|_2 &= \left\| \begin{pmatrix}
\gamma_1 (T_{a_1}^1 - \hat{T}_{a_1}^1) & \gamma_2 (T_{a_1}^2 - \hat{T}_{a_1}^2) \\
\vdots & \vdots \\
\gamma_1 (T_{a_{|\A|}}^1 - \hat{T}_{a_{|\A|}}^1) & \gamma_2 (T_{a_{|\A|}}^2 - \hat{T}_{a_{|\A|}}^2)
\end{pmatrix} \right\|_2 \\
    &\leq \sqrt{2}\max(\gamma_1, \gamma_2) \max \left( \left\| \begin{pmatrix}
(T_{a_1}^1 - \hat{T}_{a_1}^1) \\
\vdots \\
T_{a_{|\A|}}^1 - \hat{T}_{a_{|\A|}}^1
\end{pmatrix} \right\|_2, \left\| \begin{pmatrix}
(T_{a_1}^2 - \hat{T}_{a_1}^2) \\
\vdots \\
T_{a_{|\A|}}^2 - \hat{T}_{a_{|\A|}}^2
\end{pmatrix} \right\|_2 \right) \\
&\leq \sqrt{2|\A|}\max(\gamma_1, \gamma_2) \epsilon.
\end{align*}

Therefore, $\sigma_2(A) \geq \sigma_2(\hat{A}) - \sqrt{2|\A|}\max(\gamma_1, \gamma_2) \epsilon$, and hence $\sigma_2(A) > 0$ provided that $\sigma_2(\hat{A}) > \sqrt{2|\A|}\max(\gamma_1, \gamma_2) \epsilon$.
\subsection{Proof of \Cref{thm:bernstein}}
\begin{proof}
$\hat{T}_a$ can be constructed as follows. Sample $\frac{N}{|\S|}$ states $\{s^\prime_i\}^{\frac{N}{|\S|}}_{i=1}$  from the distribution  $T(\cdot|s, a)$ for every state $s\in\S$. Let $N(s)$ denote the number of times state $s$ has been sampled, i.e. $N(s) = \frac{N}{|\S|}$. Form the matrix $\tilde{T}_i = [\frac{\mathbf{1}(s_i=s,s^\prime_i=s^\prime)N}{N(s)}]_{s,s^\prime}$
It holds that $\forall i$, $\mathbb{E}[\tilde{T}_i] = T_a$, $\lambda_{\mathrm{max}}(\tilde{T}_i-T_a) \leq |\mathcal{S}| + 1$, $\lambda_{\mathrm{max}}(\mathbb{E}[(\tilde{T}_i - T_a)^2]) \leq  |\mathcal{S}|^2$ and $\mathrm{Trace}(\mathbb{E}[(\tilde{T}_i - T_a)^2]) \leq |\mathcal{S}|^2$. Then, the result follows applying Lemma 10 in \cite{Hsu:2012} and assuming $\delta < 1/e$. Finally, we conclude with a covering argument over the set $\A$.
\end{proof}

\subsection{Proof of Theorem~\ref{thm:generalization}}

\begin{lem} \label{lem:A.2}
The condition of equation~\ref{eq:7.2} holds if and only if $\forall v^1, v^2 \in \R^{|\S|}$ satisfying $(I - \gamma_1 T_a^1)v^1 = (I - \gamma_2 T_a^2)v^2$, $\forall a\in\A$, there exists $v^3\in\R^{|\S|}$ such that $(I - \gamma_3 T_a^3)v^3 = (I - \gamma_1 T_a^1)v^1$, $\forall a\in\A$.
\end{lem}

\begin{proof}
Denote by $A_1, A_2$ the matrices shown and the left and right hand side of equation~\eqref{eq:7.2} respectively, so that the equation reads $\text{rank}(A_1) = \text{rank}(A_2) - |\S|$, or equivalently $2|\S| - \text{rank}(A_1) = 3|\S| - \text{rank}(A_2)$. Using the rank theorem, it follows that $\text{dim}(\text{Ker}(A_1)) = \text{dim}(\text{Ker}(A_2))$, i.e.,
\begin{equation} \label{eq:A.2.1}
    \begin{split}
        \text{dim}(\{(v^1, v^2) \in \R^{2|\S|} : (I - \gamma_1 T_a^1) v^1 &= (I - \gamma_2 T_a^2) v^2 \ \forall a\in \A\}) \\
    = \text{dim}(\{(v^1, v^2, v^3) \in \R^{3|\S|} &: (I - \gamma_1 T_a^1) v^1 = (I - \gamma_2 T_a^2) v^2 = (I - \gamma_3 T_a^3) v^3 \ \forall a\in \A\}).
    \end{split}
\end{equation}

Since all matrices $I - \gamma_3 T_a^3$ are invertible for any $a\in \A$, it follows that for any $(v^1, v^2) \in \R^{2|\S|}$, there can exist at most one vector $v^3 \in \R^{|\S|}$ such that $(I - \gamma_1 T_a^1) v^1 = (I - \gamma_2 T_a^2) v^2 = (I - \gamma_3 T_a^3) v^3 \ \forall a\in \A$. We hence deduce that 

\begin{equation} \label{eq:A.2.2}
    \begin{split}
    &\text{dim}(\{(v^1, v^2, v^3) \in \R^{3|\S|} : (I - \gamma_1 T_a^1) v^1 = (I - \gamma_2 T_a^2) v^2 = (I - \gamma_3 T_a^3) v^3 \ \forall a\in \A\}) \\
    &= \text{dim}(\{(v^1, v^2) \in \R^{2|\S|} : \exists v^3\in\R^{|\S|}, (I - \gamma_1 T_a^1) v^1 = (I - \gamma_2 T_a^2) v^2 = (I - \gamma_3 T_a^3) v^3 \ \forall a\in \A\}).
    \end{split}
\end{equation}

Plugging equation~\eqref{eq:A.2.2} in ~\eqref{eq:A.2.1}, we have
\begin{equation} \label{eq:A.2.3}
    \begin{split}
        &\text{dim}(\{(v^1, v^2) \in \R^{2|\S|} : (I - \gamma_1 T_a^1) v^1 = (I - \gamma_2 T_a^2) v^2 \ \forall a\in \A\}) \\
    &= \text{dim}(\{(v^1, v^2) \in \R^{2|\S|} : \exists v^3\in\R^{|\S|}, (I - \gamma_1 T_a^1) v^1 = (I - \gamma_2 T_a^2) v^2 = (I - \gamma_3 T_a^3) v^3 \ \forall a\in \A\}).
    \end{split}
\end{equation}

Moreover, we can clearly see that
\begin{align*}
\{(v^1, v^2) \in \R^{2|\S|} &: \exists v^3\in\R^{|\S|}, (I - \gamma_1 T_a^1) v^1 = (I - \gamma_2 T_a^2) v^2 = (I - \gamma_3 T_a^3) v^3 \ \forall a\in \A\} \\
&\subseteq \{(v^1, v^2) \in \R^{2|\S|} : (I - \gamma_1 T_a^1) v^1 = (I - \gamma_2 T_a^2) v^2 \ \forall a\in \A\}.
\end{align*}

Thus, together with equation ~\eqref{eq:A.2.3}, we can conclude that 
\begin{align*}
\{(v^1, v^2) \in \R^{2|\S|} &: \exists v^3\in\R^{|\S|}, (I - \gamma_1 T_a^1) v^1 = (I - \gamma_2 T_a^2) v^2 = (I - \gamma_3 T_a^3) v^3 \ \forall a\in \A\} \\
&= \{(v^1, v^2) \in \R^{2|\S|} : (I - \gamma_1 T_a^1) v^1 = (I - \gamma_2 T_a^2) v^2 \ \forall a\in \A\}
\end{align*}
which shows the result.

Suppose now that condition~\ref{eq:7.2} does not hold, i.e.,

\begin{equation} 
    \begin{split}
        \text{dim}(\{(v^1, v^2) \in \R^{2|\S|} : (I - \gamma_1 T_a^1) v^1 &= (I - \gamma_2 T_a^2) v^2 \ \forall a\in \A\}) \\
    > \text{dim}(\{(v^1, v^2) \in \R^{3|\S|} &: \exists v^3 \in \R^{|\S|}, (I - \gamma_1 T_a^1) v^1 = (I - \gamma_2 T_a^2) v^2 = (I - \gamma_3 T_a^3) v^3 \ \forall a\in \A\}).
    \end{split}
\end{equation}

This directly implies that there must exist a pair $(v^1, v^2)$, such that there exists no $v^3 \in \R^{|\S|}$ satisfying $(I - \gamma_3 T_a^3) v^3 = (I - \gamma_1 T_a^1) v^1 \ \forall a\in \A$ hence finalizing the proof.

\end{proof}

We now turn to the proof of Theorem~\ref{thm:generalization}. Let $r^*$ be the ground truth reward, and suppose that we recover some reward function $r$ from policies $\pi^1, \pi^2$, i.e., $\pi^1, \pi^2$ are optimal with respect to both rewards $r$ and $r^*$ on $(T^1, \gamma_1), (T^2, \gamma_2)$ respectively. Suppose that we train a policy $\pi^3$ optimally with respect to $r$ on $(T^3, \gamma_3)$. We want to show that $\pi^3$ is also optimal with respect to the true reward $r^*$.


Let $v^i, v^i_*$ be the value vectors associated to expert $i=1,2$ with respect to rewards $r$ and $r^*$ respectively, i.e., such that
\begin{align}
\label{eq:A.2.4}
    r(\cdot, a) &= \lambda \log \pi^1(a | \cdot) + (I - \gamma_1T_a^1)v^1 = \lambda \log \pi^2(a | \cdot) + (I - \gamma_2T_a^2)v^2 \\
\label{eq:A.2.5}
    r^*(\cdot, a) &= \lambda \log \pi^1(a | \cdot) + (I - \gamma_1T_a^1)v^1_* = \lambda \log \pi^2(a | \cdot) + (I - \gamma_2T_a^2)v^2_*.
\end{align}

Let $v^3$ be the value vector associated with expert $3$ with respect to reward $r$, i.e., such that $\forall a\in \A$
\begin{equation} \label{eq:A.2.6}
    r(\cdot, a) = \lambda \log \pi^3(a | \cdot) + (I - \gamma_3T_a^3)v^3.
\end{equation}

We need to show that there exists a vector $v^3_* \in \R^{|\S|}$ such that $\forall a\in \A$
\begin{equation}
    r^*(\cdot, a) = \lambda \log \pi^3(a | \cdot) + (I - \gamma_3T_a^3)v^3_*.
\end{equation}

Using equations~\eqref{eq:A.2.4},~\eqref{eq:A.2.5} and ~\eqref{eq:A.2.6}, we have $\forall a\in \A$
\begin{align}
    r^*(\cdot, a) &= \lambda \log \pi^1(a | \cdot) + (I - \gamma_1 T_a^1)v^1_* \\
    &= r(\cdot, a) - (I - \gamma_1 T_a^1)v^1 + (I - \gamma_1 T_a^1)v^1_* \\
    &= \lambda \log \pi^3(a | \cdot) + (I - \gamma_3 T_a^3)v^3 + (I - \gamma_1 T_a^1)(v^1 - v^1_*).
    \label{eq:A.2.7}
\end{align}

Moreover, subtracting equations~\eqref{eq:A.2.4} and~\eqref{eq:A.2.5}, we have 
\[
(I - \gamma_1 T_a^1)(v^1 - v^1_*) = (I - \gamma_2 T_a^2)(v^2 - v^2_*)
\]

Therefore, using our assumption and Lemma~\ref{lem:A.2}, there exists a vector $\tilde{v}_3 \in \R^{|\S|}$ such that $(I - \gamma_1 T_a^1)(v^1 - v^1_*) = (I - \gamma_3 T_a^3)\tilde{v}_3$. Hence, combined with equation~\eqref{eq:A.2.7}, we conclude that there exists $v^3_* \in \R^{|\S|}$ such that $\forall a\in \A$
\begin{equation}
    r^*(\cdot, a) = \lambda \log \pi^3(a | \cdot) + (I - \gamma_3T_a^3)v^3_*.
\end{equation}

Using Theorem~\ref{thm:single_expert}, we conclude that $r^*$ belongs to the set rewards compatible with $\pi^3$, and hence that $\pi^3$, which has been optimized for $r$, is also optimal for the ground truth reward $r^*$.

On the other hand, if condition~\ref{eq:7.2} does not hold, according to Lemma~\ref{lem:A.2}, we can construct a reward function $r$ compatible with experts $1$ and $2$ that cannot be written in the form $r(\cdot, a) = \lambda \log \pi^3(a | \cdot) + (I - \gamma_3T_a^3)v^3$ for some $v^3 \in \R^{|\S|}$. Hence, thanks to Theorem~\ref{thm:single_expert}, the policy $\pi^3$ cannot be optimal for such a reward function. Hence, there will necessarily exist some recovered reward functions that would lead to a sub-optimal policy in environment $3$.

\subsection{Proof of Corollary~\ref{cor:3}} \label{app:A.8}

For the setup describe in this corollary, we need to verify condition~\eqref{eq:7.2}. By the rank theorem, this condition is equivalent to

\begin{equation}\label{eq:cor3.1}
\text{dim}\left(\text{Ker}
\begin{pmatrix}
I - \gamma_1 T_{a_1} & I - \gamma_2 T_{a_1} \\
\vdots & \vdots \\
I - \gamma_1 T_{a_{|\mathcal{A}|}} & I - \gamma_2 T_{a_{|\mathcal{A}|}}
\end{pmatrix} \right)
= \text{dim}\left(\text{Ker}
\begin{pmatrix}
I - \gamma_1 T_{a_1} & I - \gamma_2 T_{a_1} & \textbf{0} \\
\vdots & \vdots & \vdots \\
I - \gamma_1 T_{a_{|\mathcal{A}|}} & I - \gamma_2 T_{a_{|\mathcal{A}|}} & \textbf{0} \\
I - \gamma_1 T_{a_1} & \textbf{0} & I - \gamma_3 T_{a_1} \\
\vdots & \vdots & \vdots \\
I - \gamma_1 T_{a_{|\mathcal{A}|}} & \textbf{0} & I - \gamma_3 T_{a_{|\mathcal{A}|}}
\end{pmatrix} \right).
\end{equation}

To this end, we will show that any element $(v^1, v^2) \in \R^{2|\S|}$ of the kernel space of the left hand side is associated a single element $(v^1, v^2, v^3) \in \R^{3|\S|}$ of the kernel space of the right hand side. More precisely, we need to show that for any $v^1, v^2$ satisfying
\[
(I - \gamma_1 T_a) v^1 = (I - \gamma_2 T_a) v^2 \ \forall a\in \A,
\]
there exists a unique $v^3\in \R^{|\S|}$ such that 
\[
(I - \gamma_1 T_a) v^1 = (I - \gamma_3 T_a) v^3 \ \forall a\in \A.
\]

Consider the action $a_0$ satisfying by assumption that $T_{a_0}$ commutes with all other matrices $T_a$, $a\in \A$. Define $v^3 = (I - \gamma_3 T_{a_0})^{-1} (I - \gamma_1 T_{a_0}) v^1$. Notice that for any $a\in \A$, we can write $I - \gamma_3 T_a = \alpha (I - \gamma_1 T_a) + (1-\alpha)(I - \gamma_2 T_a)$ where $\alpha = \frac{\gamma_3 - \gamma_2}{\gamma_1 - \gamma_2}$. Moreover, recall that, if any two invertible matrices $A$ and $B$ commute, then $A$ and $B^{-1}$ also commute.

Using these properties, we then have for any $a \in \A$,

\begin{align*}
    (I - \gamma_3 T_a) v^3 &= \alpha(I - \gamma_1 T_a) v^3 + (1-\alpha)(I - \gamma_2 T_a) v^3 \\
    &= \alpha(I - \gamma_1 T_a) (I - \gamma_3 T_{a_0})^{-1} (I - \gamma_1 T_{a_0}) v^1 + (1-\alpha)(I - \gamma_2 T_a) (I - \gamma_3 T_{a_0})^{-1} (I - \gamma_1 T_{a_0}) v^1 \\
    &= \alpha(I - \gamma_1 T_a) (I - \gamma_3 T_{a_0})^{-1} (I - \gamma_1 T_{a_0}) v^1 + (1-\alpha)(I - \gamma_2 T_a) (I - \gamma_3 T_{a_0})^{-1} (I - \gamma_2 T_{a_0}) v^2 \\
    &= \alpha(I - \gamma_1 T_a) (I - \gamma_3 T_{a_0})^{-1} (I - \gamma_1 T_{a_0}) v^1 + (1-\alpha)(I - \gamma_3 T_{a_0})^{-1} (I - \gamma_2 T_{a_0}) (I - \gamma_2 T_a) v^2 \\
    &= \alpha(I - \gamma_1 T_a) (I - \gamma_3 T_{a_0})^{-1} (I - \gamma_1 T_{a_0}) v^1 + (1-\alpha)(I - \gamma_3 T_{a_0})^{-1} (I - \gamma_2 T_{a_0}) (I - \gamma_1 T_a) v^1 \\
    &= (I - \gamma_1 T_a) (I - \gamma_3 T_{a_0})^{-1} (\alpha (I - \gamma_1 T_{a_0}) + (1-\alpha)(I - \gamma_2 T_{a_0}))v^1 \\
    &= (I - \gamma_1 T_a) v^1.
\end{align*}

Uniqueness of $v^3$ is trivial since the matrices $(I - \gamma_3 T_a)$ are invertible, which shows that condition~\eqref{eq:cor3.1} holds.

\paragraph{Counter-example when the commutativity constraint does not hold.} We now provide a simple example showing that the required generalizability condition~\eqref{eq:7.2} does not always hold in the case where the commutativity condition breaks. Suppose $|\S| = 3$, $|\A| = 2$ and

\begin{equation}
    T_{a_1} = \begin{pmatrix}
    0.5 & 0.2 & 0.3 \\
    0.3 & 0.5 & 0.2 \\
    0 & 0.5 & 0.5
    \end{pmatrix}
    , \quad T_{a_2} = \begin{pmatrix}
    0.3 & 0.4 & 0.3 \\
    0.7 & 0.1 & 0.2 \\
    0.4 & 0.1 & 0.5
    \end{pmatrix}.
\end{equation}

These matrices do not commute and we have for any discount factors $\gamma_1, \gamma_2, \gamma_3$ all different,
\begin{equation}
4 = \text{rank}
\begin{pmatrix}
I - \gamma_1 T_{a_1} & I - \gamma_2 T_{a_1} \\
I - \gamma_1 T_{a_{2}} & I - \gamma_2 T_{a_{2}}
\end{pmatrix}
\neq \text{rank}
\begin{pmatrix}
I - \gamma_1 T_{a_1} & I - \gamma_2 T_{a_1} & \textbf{0} \\
I - \gamma_1 T_{a_{2}} & I - \gamma_2 T_{a_{2}} & \textbf{0} \\
I - \gamma_1 T_{a_1} & \textbf{0} & I - \gamma_3 T_{a_1} \\
I - \gamma_1 T_{a_{2}} & \textbf{0} & I - \gamma_3 T_{a_{2}}
\end{pmatrix} - |\S| = 5.
\end{equation}

\section{Algorithms details} \label{app:B}
This section provides the detailed pseudocode of the procedures we introduced for reward identification (\Cref{alg:identifiability}), for generalizability (\Cref{alg:generalization}) and identification when the reward function can be expressed as linear combination of known features (\Cref{alg:identifiability_linear}).
\begin{algorithm}[!h]
			\caption{\texttt{Identifiability Test}}
		\label{alg:identifiability}
			\begin{algorithmic}
				\STATE {\bfseries Input:} Expert transition matrices $T_1, T_2$, entropy-regularized optimal policies $\pi_1, \pi_2$.
				\STATE Compute matrix \begin{equation} \label{eq:alg1A} A := \begin{pmatrix}
-(I - \gamma_1 T_{a_1}^1) & I - \gamma_2 T_{a_1}^2 \\
\vdots & \vdots \\
-(I - \gamma_1 T_{a_{|\mathcal{A}|}}^1) & I - \gamma_2 T_{a_{|\mathcal{A}|}}^2 \\
\end{pmatrix}\end{equation}
				\IF{$\text{rank}(A) = 2 |\S| - 1$}
				\STATE $\mathrm{Identifiable = True}$
				\STATE Form vector $b \in \mathbb{R}^{|\S||\A|}$ such that $b(s,a) = \lambda \log \frac{\pi^1(a|s)}{\pi^2(a|s)}$ (ordered by states first)
				\STATE Recover value vectors $\begin{pmatrix}
				v^1 \\ v^2 
				\end{pmatrix}= (A^TA)^{-1}A^Tb$
				\STATE Recover the reward function as $r(s,a)  = \lambda \log \pi^1(a|s) + \gamma \sum_{s^\prime}T_1(s^\prime|s,a)v^1(s^\prime) - v^1(s) $ or equivalently $r(s,a)  = \lambda \log \pi^2(a|s) + \gamma \sum_{s^\prime}T_2(s^\prime|s,a)v^2(s^\prime) - v^2(s) $
				\ELSE \STATE $\mathrm{Identifiable = False}$
				\ENDIF
			\STATE {\bfseries Output:} $\mathrm{Identifiable}$ and recovered reward $r$.
			\end{algorithmic}
		\end{algorithm}
\begin{algorithm}[!h]
			\caption{\texttt{Identifiability Test with linear reward function}}
		\label{alg:identifiability_linear}
			\begin{algorithmic}
				\STATE {\bfseries Input:} Expert transition matrices $T^1, T^2$, entropy-regularized optimal policies $\pi_1, \pi_2$, features set $\{f_{a}\}_a$.
				\STATE Compute matrix \begin{equation} A := \begin{pmatrix}
-(I - \gamma_1 T_{a_1}^1) & I - \gamma_2 T_{a_1}^2 & \textbf{0} \\
\vdots & \vdots & \vdots \\
-(I - \gamma_1 T_{a_{|\mathcal{A}|}}^1) & I - \gamma_2 T_{a_{|\mathcal{A}|}}^2 & \textbf{0} \\
& & \\
-(I - \gamma_1 T_{a_1}^1) & \textbf{0} & f_{a_1} \\
\vdots & \vdots & \vdots \\
-(I - \gamma_1 T_{a_{|\mathcal{A}|}}^1) & \textbf{0} & f_{a_{|\A|}}
\end{pmatrix}\end{equation}
				\IF{$\text{rank}(A) = 2 |\S| + d$}
				\STATE $\mathrm{Identifiable = True}$
				\STATE Form vectors $b_1, b_2 \in \mathbb{R}^{|\S||\A|}$ defined as $b_1(s,a) = \lambda \log \frac{\pi^1(a|s)}{\pi^2(a|s)}$, $b_2(s,a) = \lambda \log \pi^1(a|s)$ and $b \in \mathbb{R}^{2|\S||\A|}$ as $b = \begin{pmatrix}
				b_1 \\ b_2 
				\end{pmatrix}$
				\STATE Recover value vectors and reward weights $\begin{pmatrix}
				v^1 \\ v^2 \\ w
				\end{pmatrix}= (A^TA)^{-1}A^Tb$
				\STATE Recover the reward function as $r(s,a)  = w^T f_{s,a} $
				\ELSE \STATE $\mathrm{Identifiable = False}$
				\ENDIF
			\STATE {\bfseries Output:} $\mathrm{Identifiable}$ and recovered reward $r$.
			\end{algorithmic}
		\end{algorithm}
\begin{algorithm}[!h]
			\caption{\texttt{Generalization Test}}
		\label{alg:generalization}
			\begin{algorithmic}
				\STATE {\bfseries Input:} Expert transition matrices $T^1, T^2$, transfer transition matrix $T^3$, entropy-regularized optimal policies $\pi_1, \pi_2$.
				\STATE Compute matrix \begin{equation*} A := \begin{pmatrix}
-(I - \gamma_1 T_{a_1}^1) & I - \gamma_2 T_{a_1}^2 \\
\vdots & \vdots \\
-(I - \gamma_1 T_{a_{|\mathcal{A}|}}^1) & I - \gamma_2 T_{a_{|\mathcal{A}|}}^2 \\
\end{pmatrix}\end{equation*}
				\IF{the condition in \Cref{eq:7.2} holds}
				\STATE $\mathrm{Generalizable = True}$
				\STATE Form vector $b \in \mathbb{R}^{|\S||\A|}$ such that $b(s,a) = \lambda \log \frac{\pi^1(a|s)}{\pi^2(a|s)}$
				\STATE Recover the value vectors $\begin{pmatrix}
				v^1 \\ v^2 
				\end{pmatrix}= (A^TA)^{-1}A^Tb$
				\STATE Recover the reward function as $r(s,a)  = \lambda \log \pi^1(a|s) + \gamma \sum_{s^\prime}T_1(s^\prime|s,a)v^1(s^\prime) - v^1(s) $
				\STATE Recover the optimal entropy regularized policy $\pi^3$ in $T^3$ using the recovered  reward $r$ with any RL algorithm.
				\ELSE \STATE $\mathrm{Generalizable = False}$
				\ENDIF
			\STATE {\bfseries Output:} $\mathrm{Generalizable}$ and recovered policy $\pi^3$.
			\end{algorithmic}
		\end{algorithm}
		
Algorithms \ref{alg:identifiability} and \ref{alg:generalization} can be generalized to an arbitrary number of experts. Indeed, denoting the matrix in \Cref{eq:alg1A} as $A_2$, we can construct the matrix $A_n$ for $n$ experts recursively as follows:
\begin{equation} A_n := \begin{pmatrix}
\multicolumn{2}{c}{A_{n-1}} & \textbf{0} \\ & & \\
-(I - \gamma_1 T_{a_1}^1) & \textbf{0} & I - \gamma_n T_{a_1}^n \\
\vdots & \vdots & \vdots \\
-(I - \gamma_1 T_{a_{|\mathcal{A}|}}^1) & \textbf{0} & I - \gamma_n T_{a_{|\mathcal{A}|}}^n
\end{pmatrix}\end{equation}
Similarly, we can construct the vector $b_n$ as 
\begin{equation}
b_n := \begin{pmatrix}
b_{n-1} \\
\lambda \log \frac{\pi^1(a|s)}{\pi^n(a|s)}
\end{pmatrix}\end{equation}
where $b_1$ denotes the vector defined in the algorithms for $2$ experts. The rest of the procedures remain unchanged.

\section{Additional experiments}
\label{appendix:experiments}
This section provides the experimental results and environment details omitted from the main text. 

\paragraph{Additional details for \texttt{Gridworld}} In the main text, we omitted the description of the reward function. We provide it hereafter for completeness. The reward function is obtained assigning a value at every state according to the grid shown in \Cref{fig:gridworld}. This reward function would depend only on states. To obtain a state-action dependent reward function, we add a penalty of $-30$ for moving right, $-20$ for moving down, $-10$ for moving left and $0$ for a step upwards.

\paragraph{Additional details for \texttt{WindyGridworld}} In \texttt{WindyGridworld}, the agent moves of one step according to the next state sampled from $T_{\alpha}(s^\prime|s,a) = (1 - \alpha) T_{\mathrm{det}}(s^\prime|s,a) + \alpha U(s^\prime|s,a)$  where $T_{\mathrm{det}}(s^\prime|s,a)$ as in \texttt{Gridworld}. In addition to that the agent takes an additional step according to the wind direction. The wind direction $w$ is sampled from the wind distribution generated by sampling each entry of the non normalized $P_{\mathrm{wind}}$ from a normal distribution and normalizing the obtained vector. After sampling the wind direction we sample the corresponding next state from $T_{\mathrm{det}}(s^\prime|s,w)$.

The reward function is the same used for the environment \texttt{Gridworld}.

\paragraph{Results on \texttt{Random-Matrices}} We report in \Cref{fig:random-matrices}, the results omitted from the main text. In \Cref{fig:random-matrices}, we show the reward recovered with \Cref{alg:identifiability} and the difference with respect to the true reward. It clearly emerges that the recovered reward is within a constant shift from the true reward function.

\paragraph{Results on \texttt{Gridworld} with state only reward } We provide an additional result on \texttt{Gridworld} where we do not consider the penalty assigned to the different actions. In this case, the reward depends only on states but the learner is not informed about this feature. In \Cref{fig:gridworld}, we show the recovered reward. Given that the reward depends only on the states we show the 2D representation of the state space suppressing the action dimension. We used the \texttt{Gridworld} implementation released with \cite{viano2021robust}.

\paragraph{Results on \texttt{WindyGridworld} with different discount factors} In the main text we showed that we need to observe $4$ experts to generalize to a new wind distribution. Hereafter, we provide experiments on the generalization to a new environment with a different discount factor. We verified that in this case observing two experts is enough to generalize. The comparison between recovered rewards and policies can be found in \Cref{fig:gammawindygridworld}. 

\paragraph{Computational resources} The experiment can be reproduced with a standard laptop.
\begin{figure}[!h]
    \centering
    \begin{tabular}{ccc}
\subfloat[$r$]{%
       \includegraphics[width=0.2\linewidth]{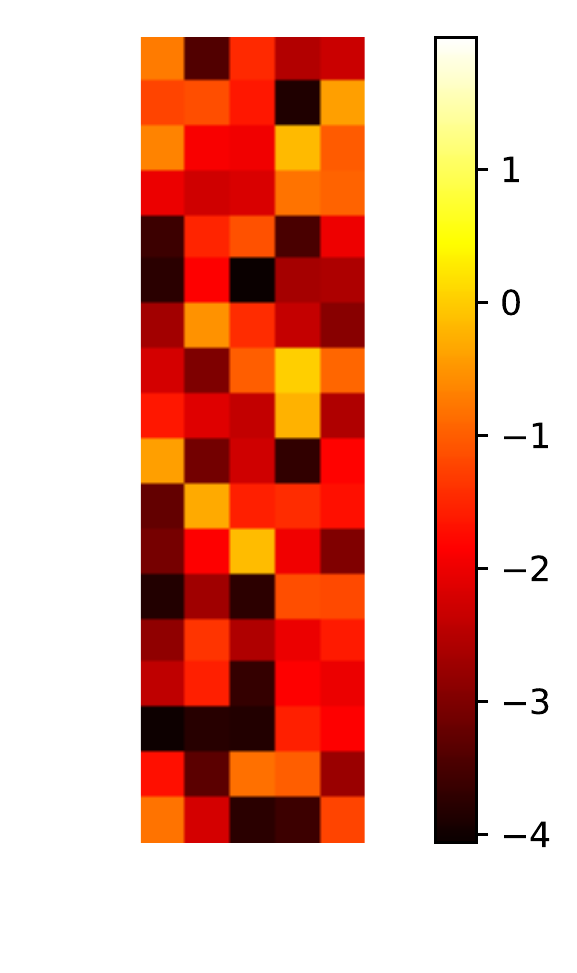}
     } &
\subfloat[ $r_\mathrm{true}$]{%
       \includegraphics[width=0.2\linewidth]{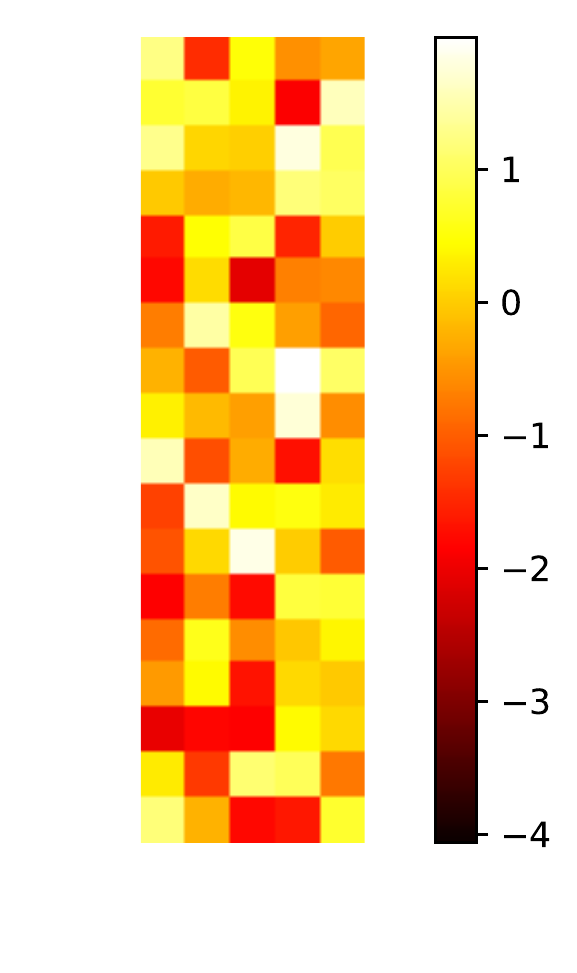}
     } &
\subfloat[$r - r_{\mathrm{true}}$]{%
       \includegraphics[width=0.2\linewidth]{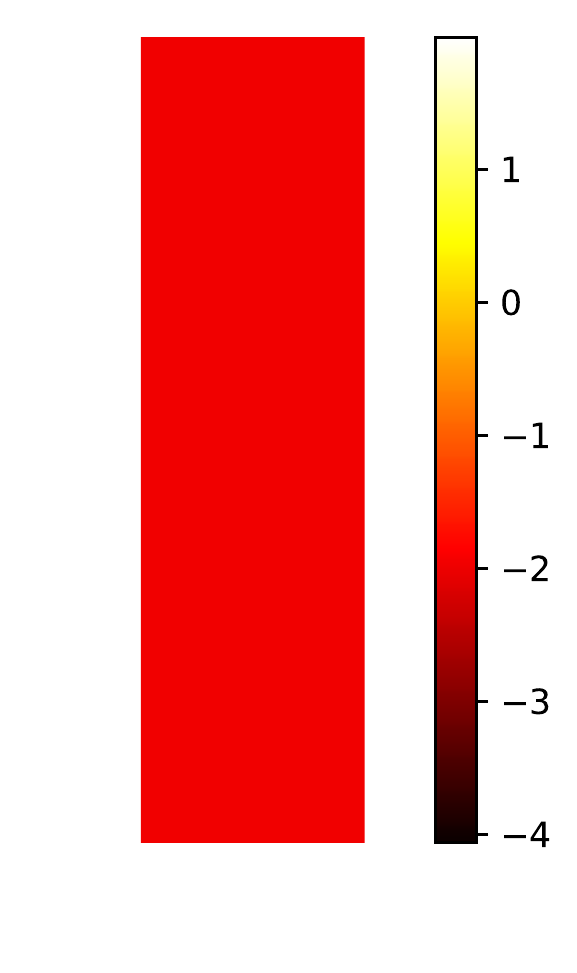}
     } \\
\end{tabular}
    \caption{Comparison between true and recovered reward in \texttt{Random-Matrices} with $|\mathcal{S}|=18$ and $|\mathcal{A}|=5$. On the vertical axis corresponds to the canonical ordering of the $18$ states while the horizontal axis corresponds to the $5$ actions.}
\label{fig:random-matrices}
\end{figure}

\begin{figure}[!h]
    \centering
    \begin{tabular}{ccc}
\subfloat [$r$]{%
       \includegraphics[width=0.2\linewidth]{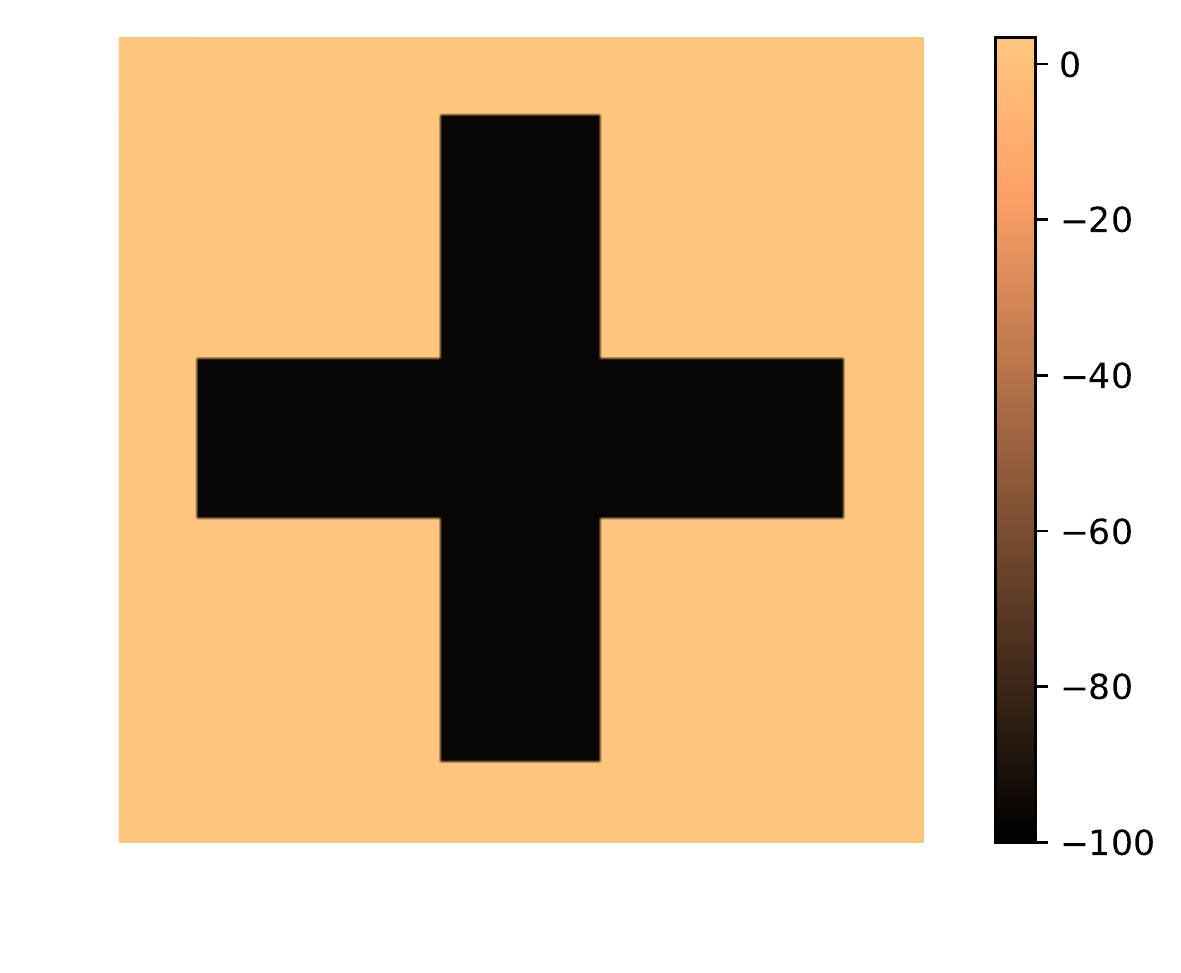}
     } &
\subfloat[ $r_\mathrm{true}$]{%
       \includegraphics[width=0.2\linewidth]{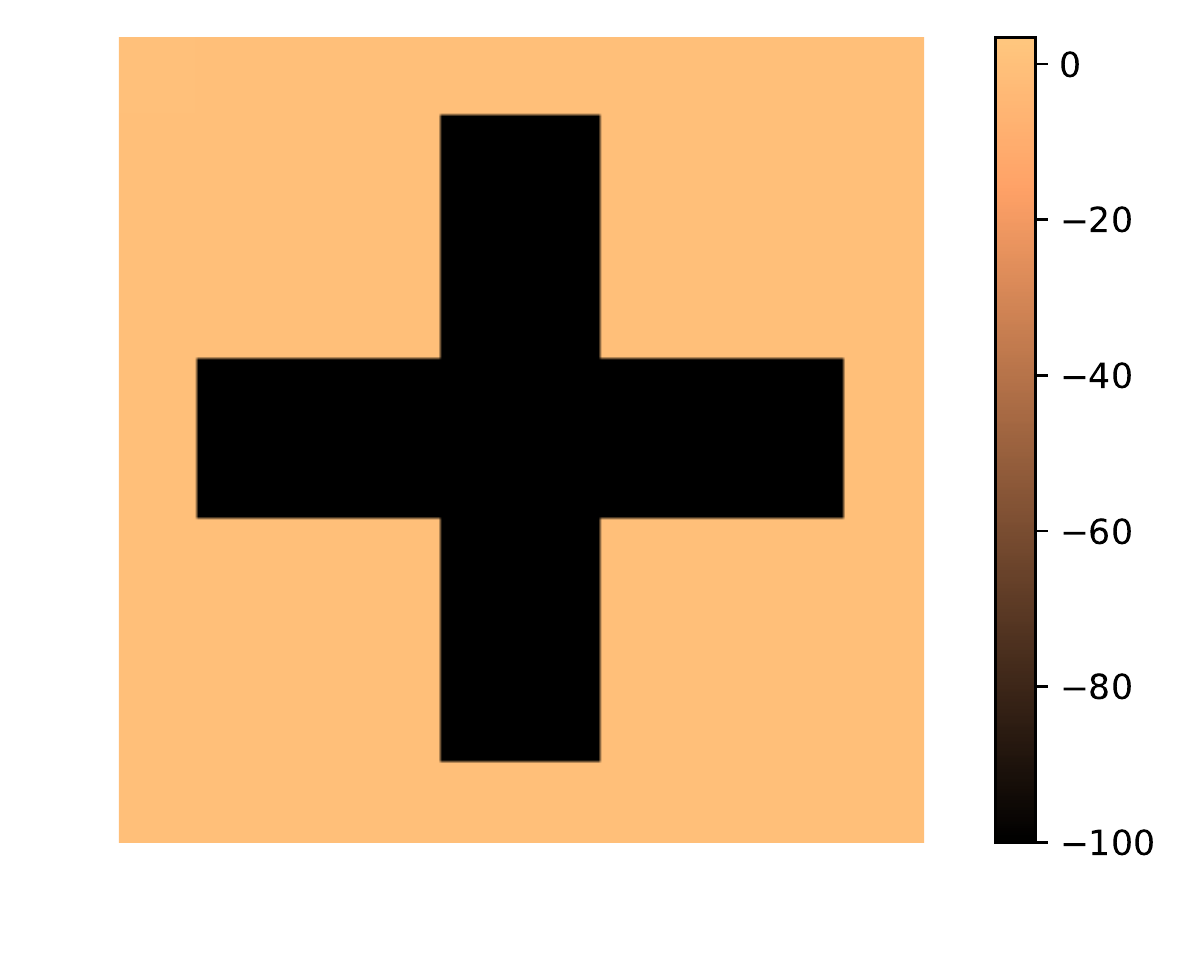}
     } &
\subfloat[$r - r_{\mathrm{true}}$]{%
       \includegraphics[width=0.2\linewidth]{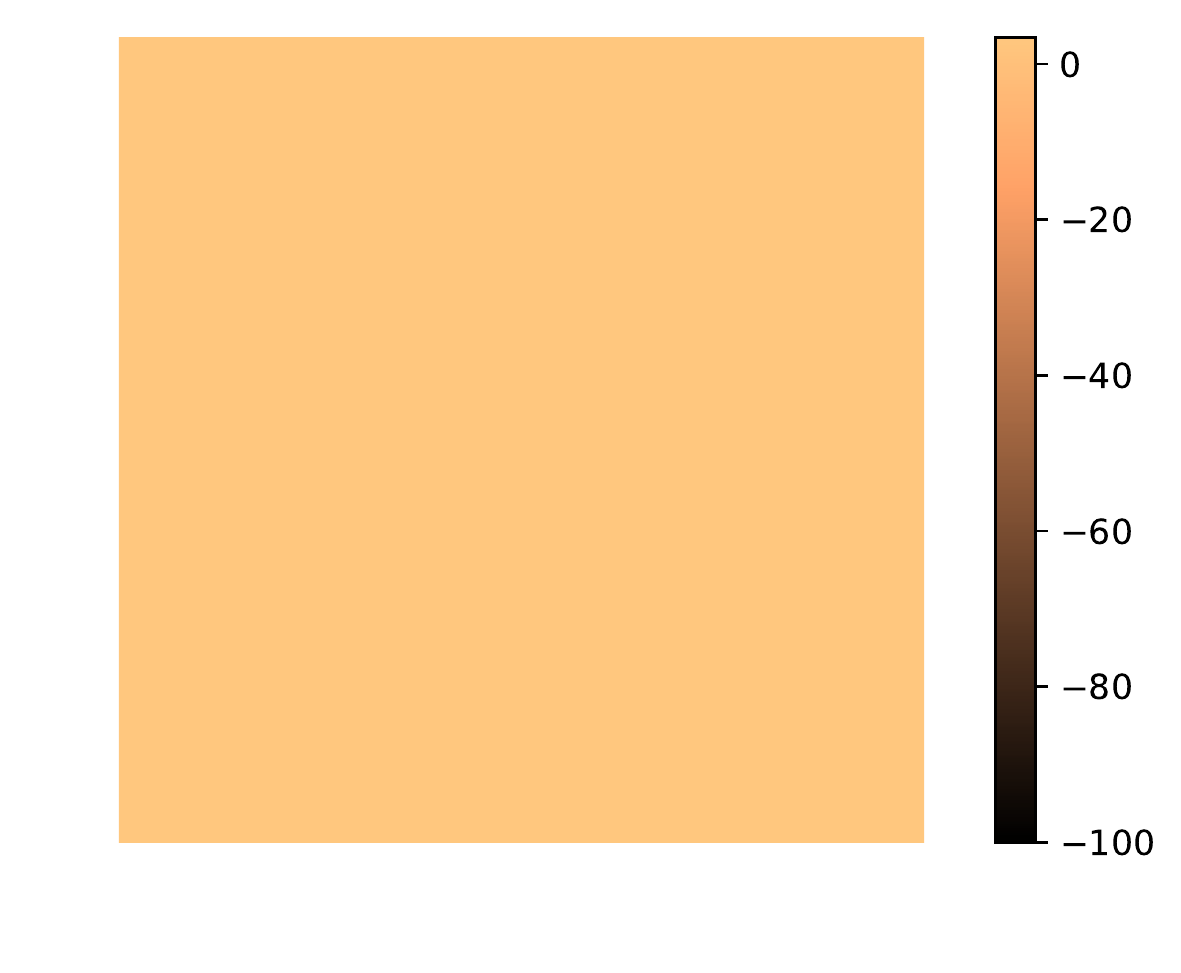}
     } \\
\end{tabular}
    \caption{Comparison between true and recovered reward in \texttt{Gridworld} with $|\mathcal{S}|=100$ and the $4$ actions up, down, left and right. It can be noticed that the reward function is recovered up to a constant shift.}
\label{fig:gridworld}
\end{figure}

\begin{figure}
    \centering
    \begin{tabular}{ccccc}
\subfloat[ $r$]{%
       \includegraphics[width=0.16\linewidth]{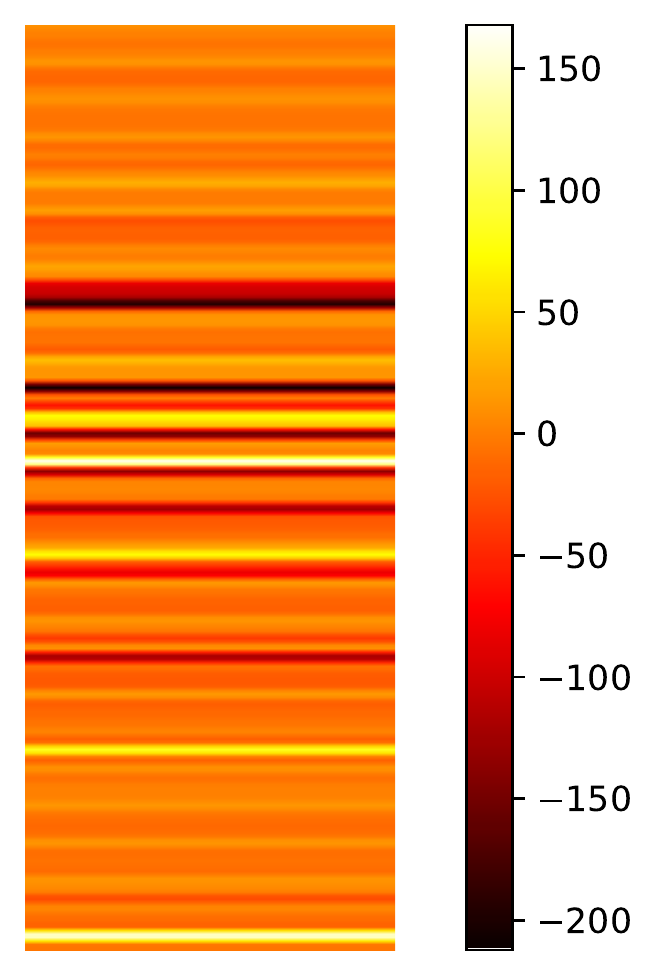}
     } &
\subfloat[ $r_\mathrm{true}$]{%
       \includegraphics[width=0.16\linewidth]{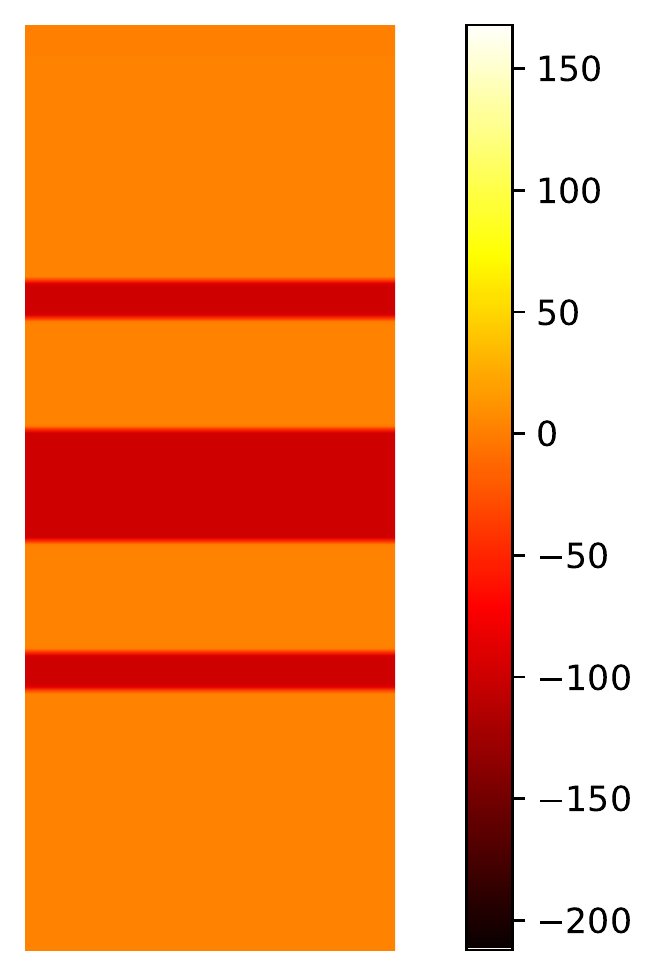}
     } &
\subfloat[$r - r_{\mathrm{true}}$]{%
       \includegraphics[width=0.16\linewidth]{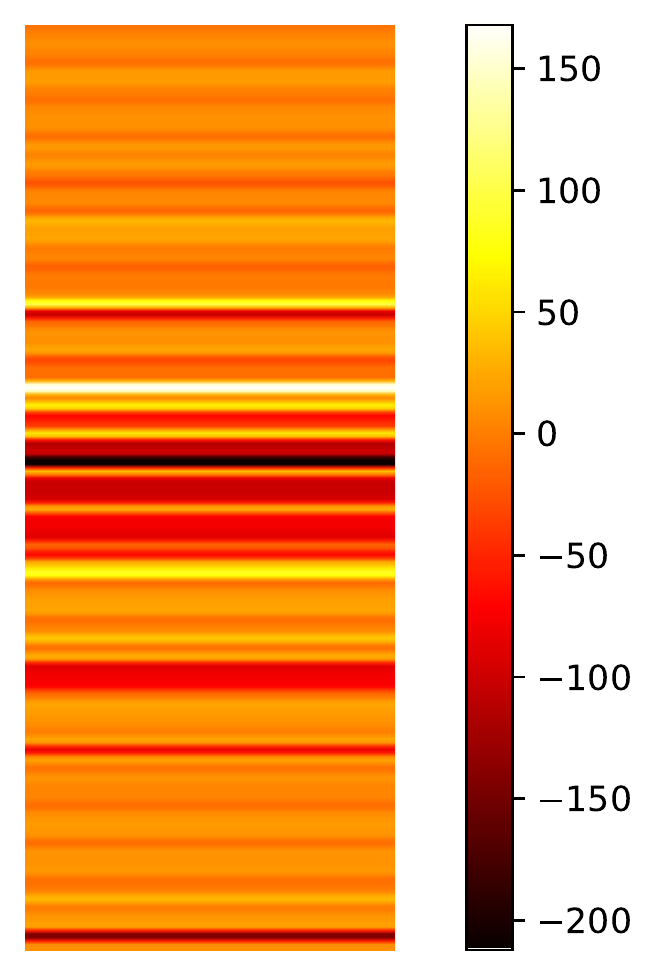}
     } &
\subfloat[$\pi^{T_{\mathrm{test}}}_{r_{\mathrm{true}}}$]{%
       \includegraphics[width=0.16\linewidth]{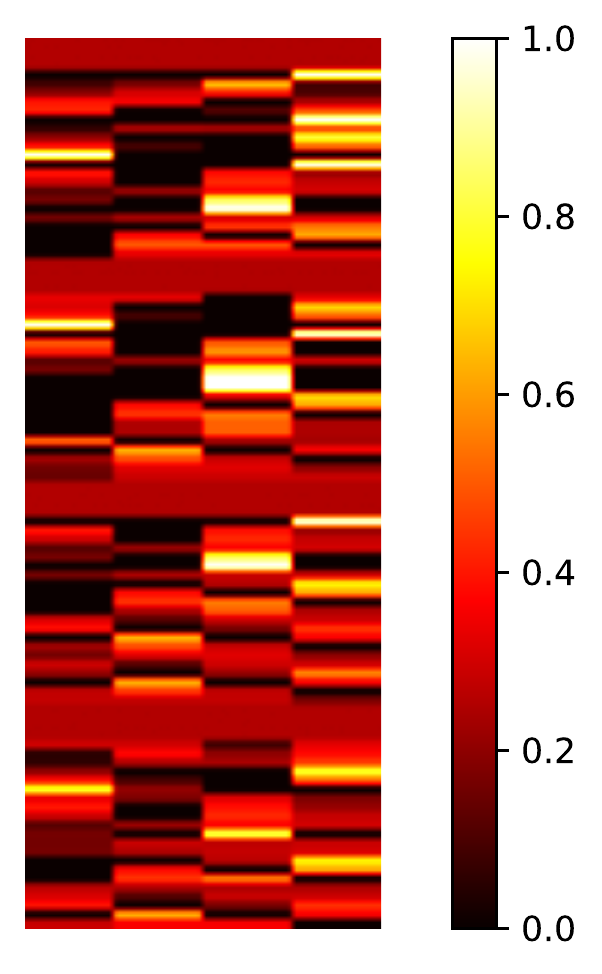}
     } &
\subfloat[$\pi^{T_{\mathrm{test}}}_{r_{\mathrm{true}}} - \pi^{T_{\mathrm{test}}}_r$]{%
       \includegraphics[width=0.16\linewidth]{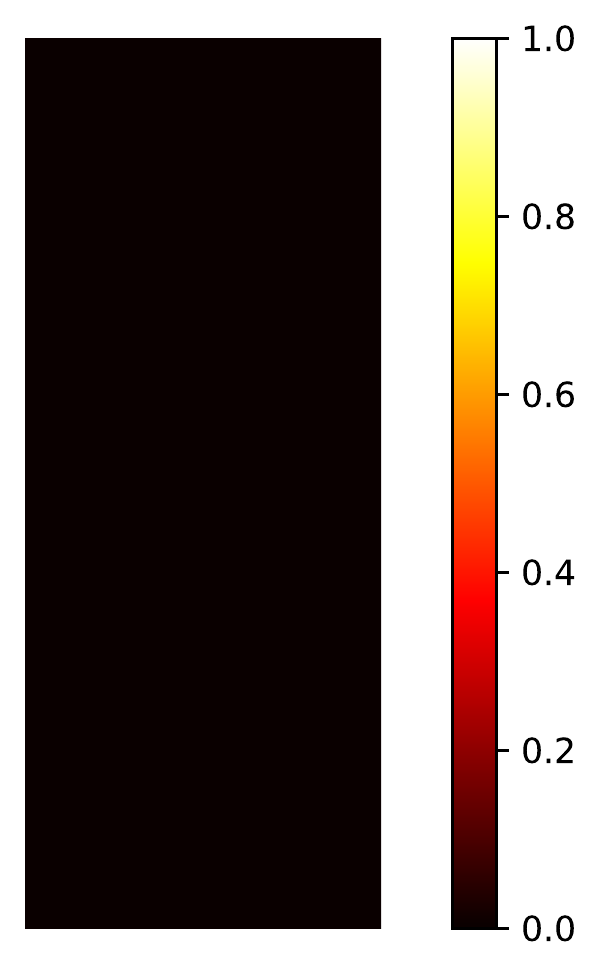}
    } \\
\end{tabular}
    \caption{ Generalization in \texttt{WindyGridworld} with different discount factors. We observe two experts with discounts factor $\gamma_1$ and $\gamma_2$ with $\gamma_1 \neq \gamma_2$ and with common transition dynamics. Subplot (e) shows that the policy recovered from $r_{\mathrm{true}}$ in a new environment with a different $\gamma_3$ matches the policy obtained from the recovered reward.}
\label{fig:gammawindygridworld}
\end{figure}

\begin{figure}
    \centering
    \begin{tabular}{ccccc}
\subfloat[ $r$]{%
       \includegraphics[width=0.16\linewidth]{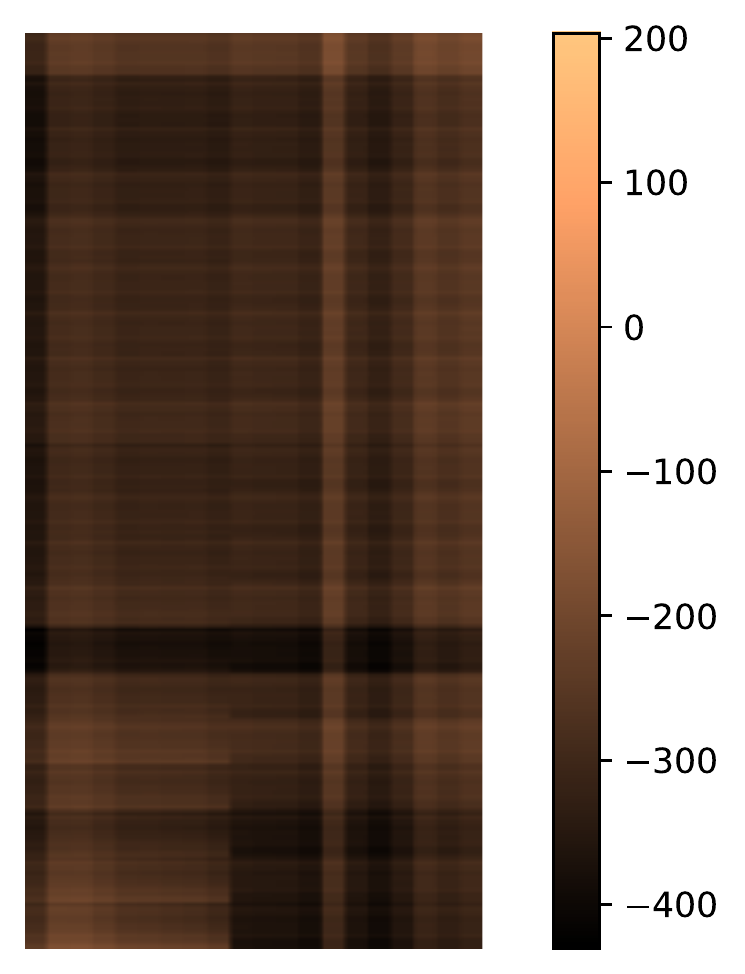}
     } &
\subfloat[ $r_\mathrm{true}$]{%
       \includegraphics[width=0.16\linewidth]{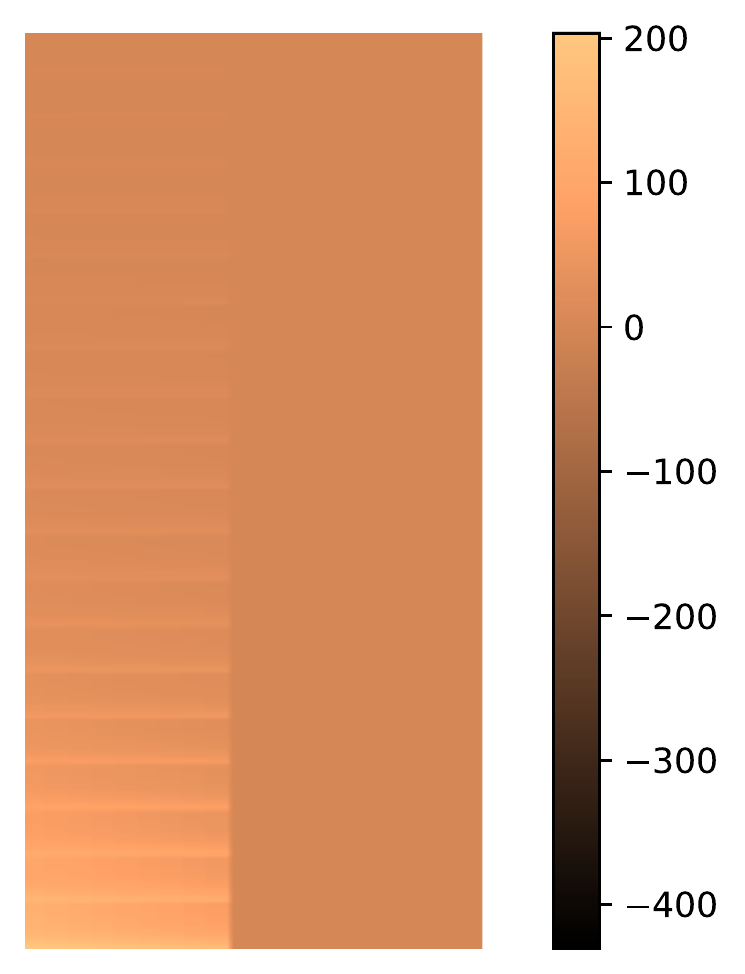}
     } &
\subfloat[$r - r_{\mathrm{true}}$]{%
       \includegraphics[width=0.16\linewidth]{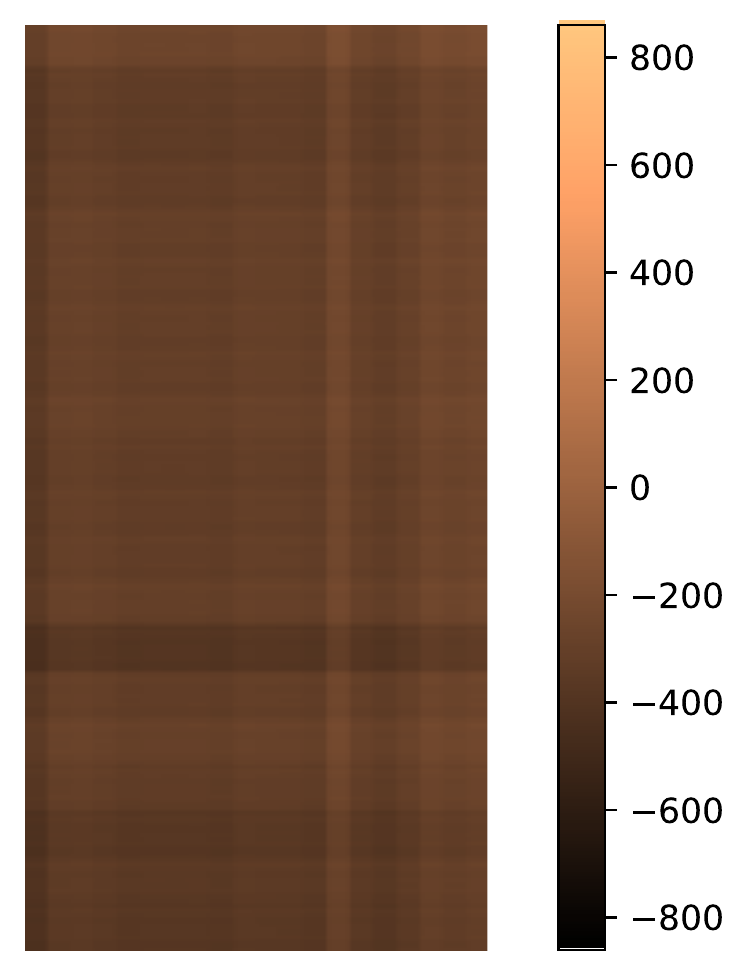}
     } &
\subfloat[$\pi^{T_{\mathrm{test}}}_{r_{\mathrm{true}}}$]{%
       \includegraphics[width=0.16\linewidth]{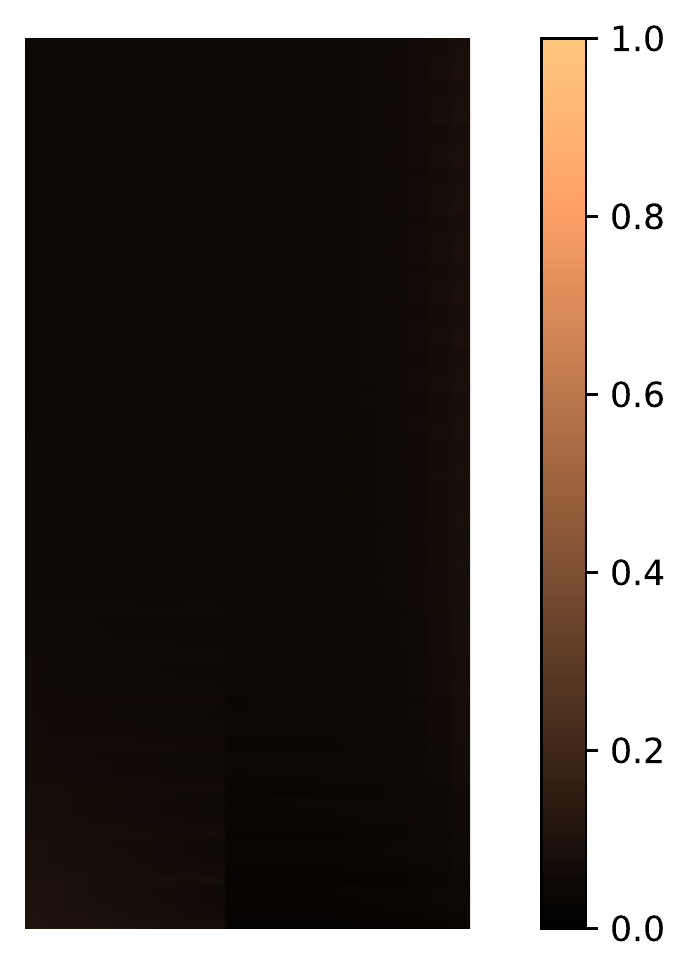}
     } &
\subfloat[$\pi^{T_{\mathrm{test}}}_{r_{\mathrm{true}}} - \pi^{T_{\mathrm{test}}}_{r}$\label{fig:strebulaev_e}]{%
       \includegraphics[width=0.16\linewidth]{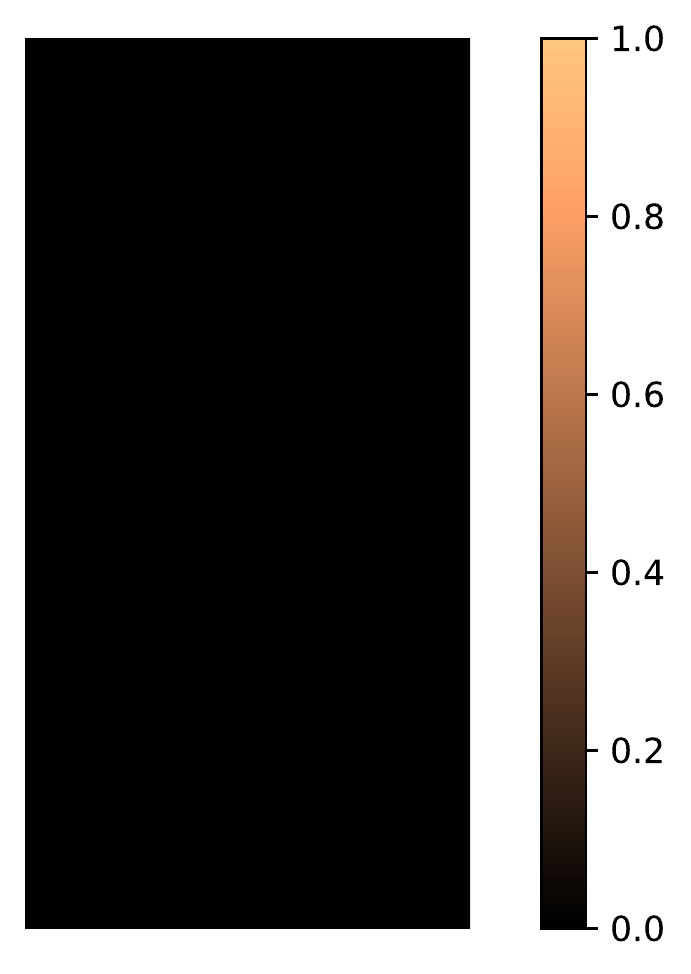}
    } \\
\end{tabular}
    \caption{Comparison between true and recovered reward in \texttt{Strebulaev-Whited} with $|\mathcal{S}|=400$ and the $20$ actions. It can be clearly noticed that the reward function is not identified (see subplots (a), (b), (c)). However, when we use the recovered reward in subplot (a) to train an optimal policy under unseen dynamics we recover the optimal policy under the true reward in subplot (b). The subplots (d) show the policies recovered from the true reward and (e) shows the difference between the policy recovered from $r_{\mathrm{true}}$ and from the recovered reward.}
\label{fig:strebulaev}
\end{figure}

\begin{figure}
    \centering
    \begin{tabular}{ccc}
\subfloat[$r$]{%
       \includegraphics[width=0.2\linewidth]{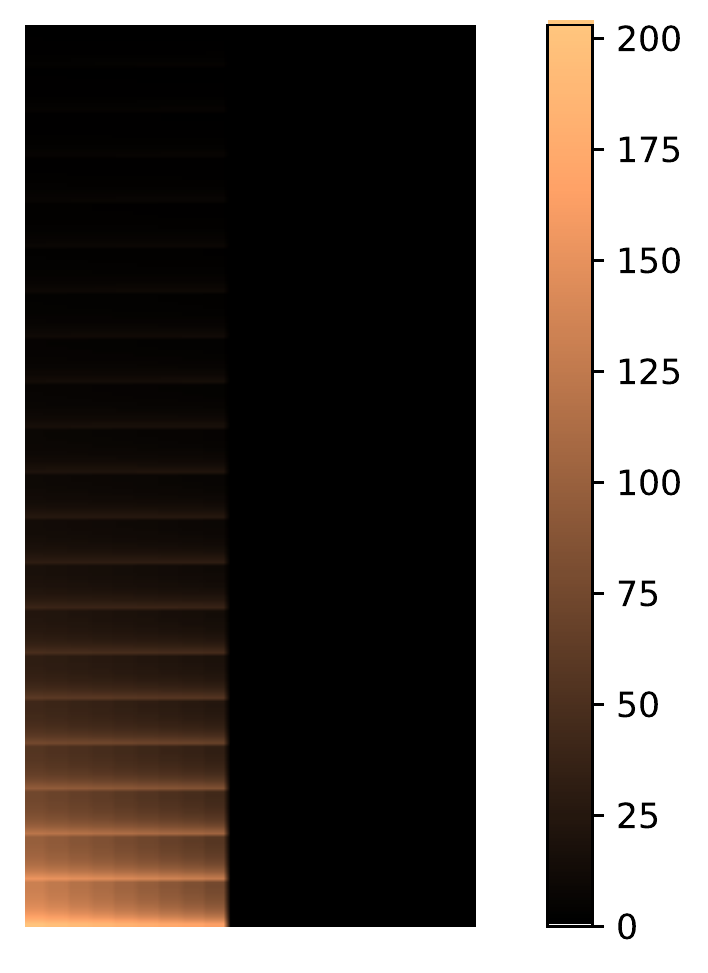}
     } &
\subfloat[ $r_\mathrm{true}$]{%
       \includegraphics[width=0.2\linewidth]{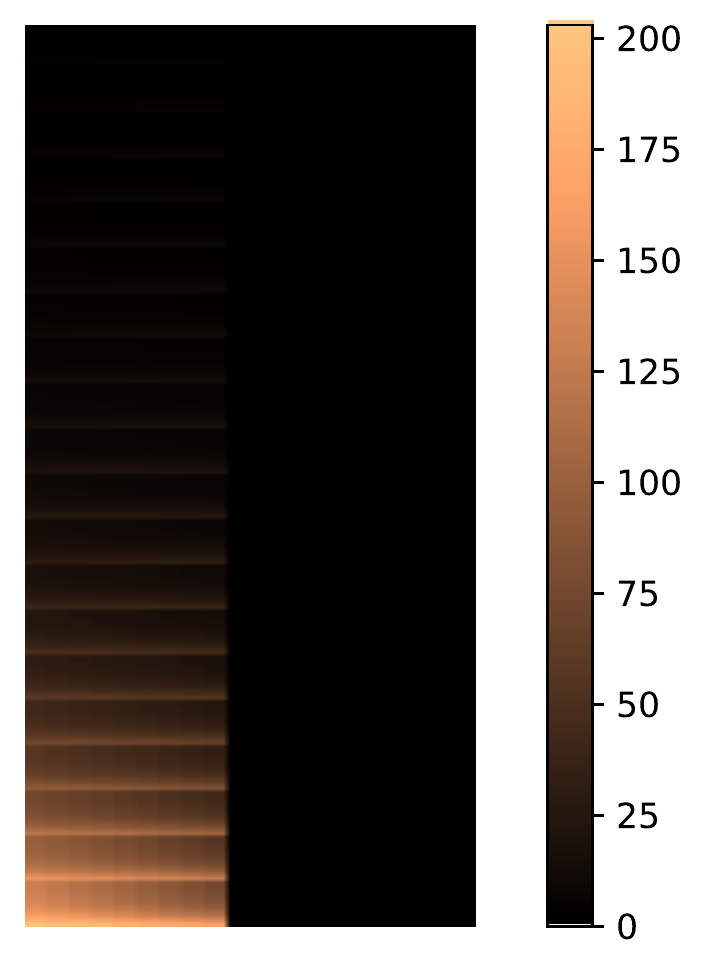}
     } &
\subfloat[$r - r_{\mathrm{true}}$]{%
       \includegraphics[width=0.2\linewidth]{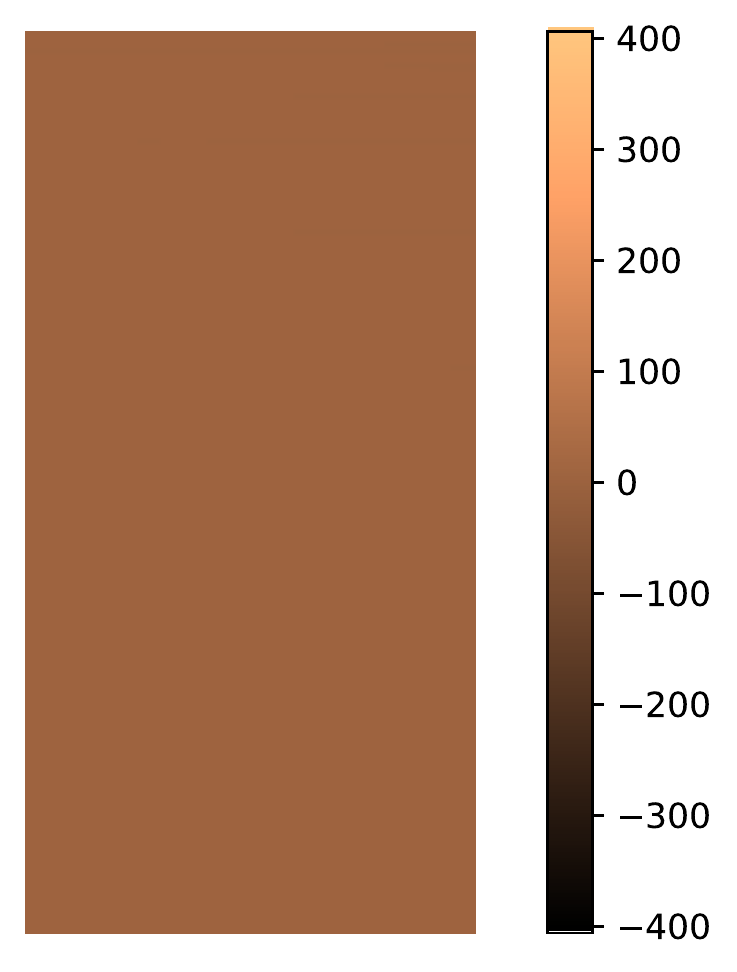}
     } \\
\end{tabular}
    \caption{Comparison between true and recovered reward in \texttt{Strebulaev-Whited} assuming additional knowledge of the features $\{f_a\}_a$. It emerges that thanks to such additional information the reward function is identifiable.}
\label{fig:lin_strebulaev}
\end{figure}
\newpage
\section{Extension to Regularized MDPs}
\label{app:regularized}
It turns out that our identifiability result is valid more generally for regularized MDPs \cite{Geist:2019} where the entropy term in equation \eqref{eq:1} is replaced by any other strongly convex differentiable function of the policy $\Omega(\pi)$.

Indeed, we can use Proposition 1 and Definition 1 in \cite{Geist:2019} to establish that for any value vector 
any $v$ and reward $r$, there exists a unique policy that satisfies
\[
\pi(a|s) = \nabla \Omega^\ast (r(s,a) + \gamma \sum_{s' \in \S} T(s'|s,a) v(s'))
\]
where $\Omega^\ast$ denotes the Fenchel conjugate of $\Omega$. By the distributivity property (iii) in Proposition 1 of [5], we can subtract a function dependent only on state in the argument without affecting the equality. This gives that for any $v$ and $r$, there exists a unique $\pi$ such that

\[
\pi(a|s) = \nabla \Omega^\ast (r(s,a) + \gamma \sum_{s' \in \S} T(s'|s,a) v(s') - v(s))
\]

Using the convexity of $\Omega$, we have that $\nabla \Omega$ is the inverse map of $\nabla \Omega^\ast$
. Hence we obtain
\[
\nabla \Omega (\pi(a|s)) = r(s,a) + \gamma \sum_{s' \in \S} T(s'|s,a) v(s') - v(s)
\]
which is the equivalent of our Theorem \ref{thm:single_expert} for general strongly convex regularizers. The only part changing is the left hand side. However, we saw in the analysis that reward identifiability was not depending on this part of the equation. When using a different regularizer, the recovered reward given observed expert policies will be different, but the identifiability condition remains the same.

This extension relax the assumption of entropy regularized experts but, unfortunately, epsilon-greedy or deterministic greedy policies would not fit this setting. Identifiability is more challenging with these kinds of experts because the knowledge of such policies only informs us with the action yielding the highest expected value, but no information about the relative difference with respect to other actions, in contrast with regularized stochastic policies.
\end{document}